\documentclass[pdflatex,sn-mathphys-num]{sn-jnl}

\usepackage{graphicx}%
\usepackage{multirow}%
\usepackage{amsmath,amssymb,amsfonts}%
\usepackage{amsthm}%
\usepackage{mathrsfs}%
\usepackage[title]{appendix}%
\usepackage{xcolor}%
\usepackage{textcomp}%
\usepackage{manyfoot}%
\usepackage{booktabs}%
\usepackage{algorithm}%
\usepackage{algorithmicx}%
\usepackage{algpseudocode}%
\usepackage{listings}%

\usepackage{siunitx}
\usepackage{lineno}
\theoremstyle{thmstyleone}%
\newtheorem{lem}{Lemma}[section]
\newtheorem{Thm}{Theorem}[section]
\newtheorem{Def}{Definition}[section]

\theoremstyle{thmstyletwo}%

\theoremstyle{thmstylethree}%

\begin{document}

\title[Article Title]{Maximum-Entropy Analog Computing Approaching ExaOPS-per-Watt Energy-efficiency at the RF-Edge}


\author[1]{\fnm{Aswin} \sur{Undavalli}}\email{a.undavalli@wustl.edu}
\equalcont{These authors contributed equally to this work.}

\author[2]{\fnm{Kareem} \sur{Rashed}}\email{rashedk@oregonstate.edu}
\equalcont{These authors contributed equally to this work.}

\author[1]{\fnm{Zhili} \sur{Xiao}}\email{xiaozhili@wustl.edu}
\equalcont{These authors contributed equally to this work.}

\author[2]{\fnm{Arun} \sur{Natarajan}}\email{Arun.Natarajan@oregonstate.edu}

\author*[1]{\fnm{Shantanu} \sur{Chakrabartty}}\email{shantanu@wustl.edu}

\author*[3]{\fnm{Aravind} \sur{Nagulu}}\email{a.nagulu@northeastern.edu}

\affil[1]{Department of Electrical and Systems Engineering, Washington University in St. Louis, Saint Louis, MO}
\affil[2]{Department of Electrical Engineering and Computer Science, Oregon State University, Corvallis, OR}
\affil[3]{Department of Electrical and Computer Engineering, Northeastern University, Oakland, CA}





\abstract{In this paper, we demonstrate how the physics of entropy production, when combined with symmetry constraints, can be used for implementing high-performance and energy-efficient analog computing systems. At the core of the proposed framework is a generalized maximum-entropy principle that can describe the evolution of a mesoscopic physical system formed by an interconnected ensemble of analog elements, including devices that can be readily fabricated on standard integrated circuit technology. We show that the maximum-entropy state of this ensemble corresponds to a margin-propagation (MP) distribution and can be used for computing correlations and inner products as the ensemble's macroscopic properties.
Furthermore, the limits of computational throughput and energy efficiency can be pushed by extending the framework to non-equilibrium or transient operating conditions, which we demonstrate using a proof-of-concept radio-frequency (RF) correlator integrated circuit fabricated in a 22 nm SOI CMOS process. The measured results show a compute efficiency greater than 2 Peta ($10^{15}$) Bit Operations per second per Watt (PetaOPS/W) at 8-bit precision and greater than 0.8 Exa ($10^{18}$) Bit Operations per second per Watt (ExaOPS/W) 
at 3-bit precision for RF data sampled at rates greater than 4 GS/s. Using the fabricated prototypes, we also showcase several real-world RF applications at the edge, including spectrum sensing, and code-domain communications.}




\maketitle

\section{Introduction}\label{sec:intro}

Advances in artificial intelligence and machine learning have created opportunities for new computing platforms that can trade off computational precision for speed and energy efficiency~\cite{Mark2014Energy}. This trend has led to the resurgence of analog computing techniques, which exploit primitives such as charge or current conservation and the continuous-time dynamics of the substrate to achieve high throughput and energy efficiency. Unlike digital systems that rely on discrete and clocked logic, analog computing systems operate at the speed of physical processes, enabling inherently parallel and energy-efficient computation~\cite{chakrabartty_jssc2007}. In a conventional analog computing architecture, atomic operations like multiplications are implemented using physical multiplicative processes such as Ohm's law~\cite{Danial_NatElec2019, Wan_Nature2022}, translinear principles~\cite{Kim_JSSC2022, chakrabartty_jssc2007}, or capacitor charging~\cite{Li_JSSC2024}. These multiplier units are then physically connected together in a crossbar configuration (shown in Fig.~\ref{fig:entropy_fig1}a for two vectors ${\bf x}$ and ${\bf y}$) to implement vector operations like inner products~\cite{Kim_JSSC2022} or correlations~\cite{Li_JSSC2024}. While this modular design paradigm forms the core of many state-of-the-art analog machine learning and neuromorphic hardware architectures~\cite{Chakraborty_ProcIEE2020}, the approach has several scaling limitations. First, computing using these physical multipliers limits the choice to specific devices, such as memristors~\cite{Danial_NatElec2019}, transistors in the sub-threshold or triode regimes~\cite{Kim_JSSC2022, chakrabartty_jssc2007}, and digitally tunable capacitors~\cite{Li_JSSC2024}. Second, the array's speed, energy-efficiency, and scalability are limited by the interconnect bandwidth, cross-talk between the physical multipliers and the settling-time of the peripheral read-out circuits~\cite{Amin2022ISCASParasiticsPartitioning}. 

In this paper, we propose an analog computing paradigm where the entire array, including the interconnects, functions as a mesoscopic physical system, utilizing only statistical laws like the physics of entropy production for computation~\cite{Landauer1961EntropyProduction,peliti2021stochastic}. The inputs ${\bf x}$ and ${\bf y}$ enforces the boundary conditions of the mesostate ensemble, as depicted in Fig.~\ref{fig:entropy_fig1}b, which could be formed by the energy-barriers due to transistors and the strength of the coupling between the energy-barriers is determined by the interconnects and the parasitic capacitances.
As shown in Fig.~\ref{fig:entropy_fig1}c, each individual mesostate are collections of microstates (for example electron configurations or current paths) and for this work we will assume that the stochastic dynamics of the microstates (as shown in Fig.~\ref{fig:entropy_fig1}c) operate at a much faster time-scale than the mesoscopic time-scales~\cite{peliti2021stochastic}. Under the influence of the boundary constraints, the equilibrium or the maximum-entropy configuration of the mesostates encodes the computation of interest. In this work we show that when rectangular symmetry constraints are enforced on the inputs, a macroscopic ensemble potential tracks a monotonic function $G({\bf x}^T{\bf y})$ of the inner-product or the correlation, as shown in Fig.~\ref{fig:entropy_fig1}b. While this approach is general enough to be applied to many mesoscopic physical ensembles, in this paper, we demonstrate the proof-of-concept and associated performance benefits using a custom silicon integrated circuit, also shown in Fig.~\ref{fig:entropy_fig1}b. Specifically, we leverage the result that a Margin Propagation (MP) computing framework~\cite{Ming_TCAS2012} is a maximum-entropy configuration of a mesoscopic ensemble with a rectangular symmetry boundary conditions. The symmetry is illustrated in Fig.~\ref{fig:entropy_fig1}c for two ensembles each comprising of four potential wells $E^{\pm}_{1-2}$ and $E^{\pm}_{3-4}$, whose heights are determined by inputs $\{x_1, y_1\}$ and $\{x_2, y_2\}$. The rectangular boundary conditions for each ensemble is set according to $(x_1+y_1, -x_1-y_1, x_2+y_2, -x_2-y_2)$ and $(x_1-y_1, -x_1+y_1, x_2-y_2, -x_2+y_2)$. The height of the potential well determines the probability of occupancy of the mesostate and is denoted by $(p_1^+, q_1^+, p_2^+, q_2^+)$ and $(p_1^-, q_1^-, p_2^-, q_2^-)$ respectively and shown in Fig.~\ref{fig:entropy_fig1}d. Note that the potential wells are shown to be separated in Fig.~\ref{fig:entropy_fig1}c, which allows us to specify discrete probabilities for each mesostate, instead of continuous distributions of the underlying microscopic configuration shown in Fig.~\ref{fig:entropy_fig1}d. These two ensembles form a part of the larger ensemble as shown in Fig.~\ref{fig:entropy_fig1}e, where we have arranged the wells labeled by $x_i + y_i, -x_i - y_i$ and $x_i - y_i, y_i - x_i$ on a 2-D grid for the sake of exposition. The corresponding probability for each of the mesostates are specified by $p_i^+, q_i^+$ and $p_i^-, q_i^-$ with $\sum_i^N p_i^+ + q_i^+ = 1$ and $\sum_i^N p_i^- + q_i^- = 1$ respectively. In addition to the rectangular symmetry constraint, we assume an underlying energy-like constraint~\cite{Tsallis1988GeneralizedEntropy,JaynesPhysRev1957}
\begin{subequations}\label{eqn:ensemble constraint}
\begin{align}
\sum_{i=1}^N ( x_i + y_i)(p_i^+ - q_i^+) = U^+- E_0, \\ \sum_{i=1}^N ( x_i - y_i)(p_i^- - q_i^-) = U^-- E_0,
\end{align}
\end{subequations}
where $U^+ \ge 0$ and $U^- \ge 0$ denote the energy of each ensemble, respectively. If, other than the symmetry and energy constraints, nothing else is assumed about the interactions among the mesostates, then, according to the maximum entropy principle~\cite{JaynesPhysRev1957}, the most appropriate choice of the steady-state ensemble distribution is the one with the highest entropy.  
\begin{figure} [htbp] 
  \centering
\includegraphics[width=1.0\textwidth]{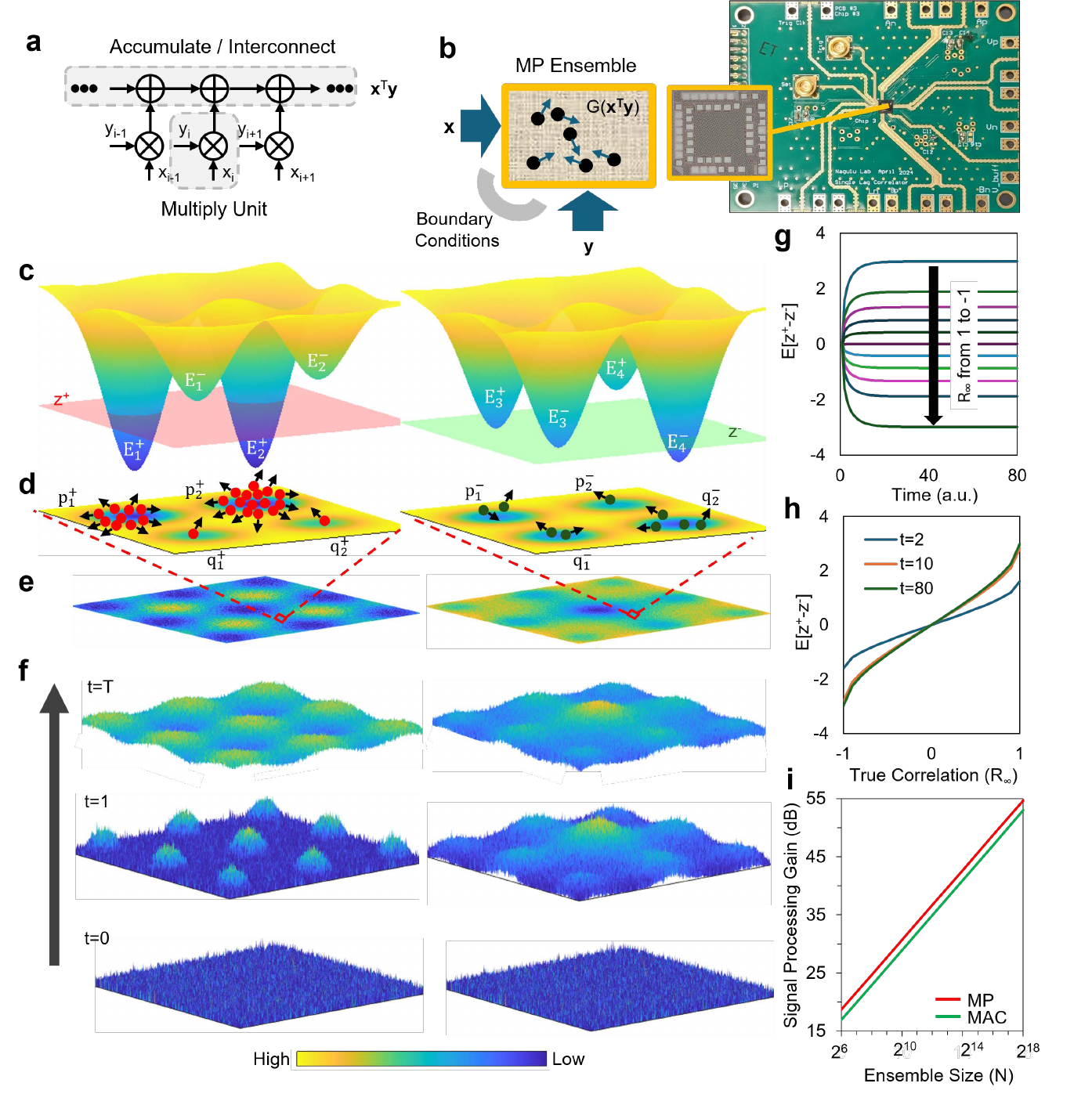}
  \caption{(a) Conventional analog computing architecture for computing inner-products and correlations using multiply-accumulate (MAC) operations. (b) Maximum-entropy analog computing architecture and an integrated circuit prototype for demonstrating the proof-of-concept. Inputs ${\bf x},{\bf y}$ are presented as boundary conditions of a statistical ensemble and the ensemble potential encodes the inner-product/correlation as $G({\bf x}^T{\bf y})$. (c) Rectangular symmetry constraints on the potential wells with respective heights ($E_1^{\pm} = E_0 \pm x_1 \pm y_1$, $E_2^{\pm} = E_0 \pm x_2 \pm y_2$, $E_3^{\pm} = E_0 \pm x_1 \mp y_1$, $E_4^{\pm} = E_0 \pm x_2 \mp y_2$). (d) Mesostates and microstates corresponding to each potential well and part of a larger ensemble (e). (f) Conceptual illustrations of a system evolving to its maximum-entropy state with under energy and symmetry constraints. The distribution of microstates within the energy wells, illustrated as red and green spheres on the heat map in (d) and denoted by $p_i^{\pm}$ for $E_1^+, E_2^+, E_3^+, E_4^+$ and $q_i^{\pm}$ for $E_1^-, E_2^-, E_3^-, E_4^-$ , will be determined by the ensemble potentials $z^+$ and $z^-$, shown by the red and green planes. 
  (g) Simulated dynamics for $h(x) = \text{max}(0,x)$ showing the expected difference between $z^+$ and $z^-$ as a function of time and across different correlation values. (h) The expected difference between $z^+$ and $z^-$ at different time steps, as the correlation is varied. (i) Comparison of the signal processing gains, SPG = 1/$\text{(rms error)}^2$ with respect to the ensemble size, for MAC-based correlators and above MP maximum entropy approach at different time steps for $t = 2, 10, 80$, where $h(x) = \text{max}(0,x)$.  
}
  \label{fig:entropy_fig1}
\end{figure}

Shannon-entropy~\cite{ShannonEntropy1948,peliti2021stochastic} would be a natural choice for measuring entropy, however, we use a more general form based on the Tsallis entropy~\cite{Havrda1967QuantificationMO,Tsallis1988GeneralizedEntropy,Tsallis2003introductionnonextensive}, which does not assume any specific information about the long-range interactions between the ensemble elements. The Tsallis entropy $S^+$ and $S^-$ for each ensemble is defined as 
\begin{equation}\label{eqn:q-functional}
S^{\pm}_{\eta} =  \frac{1 - \sum_i^N (p_i^{\pm})^\eta - \sum_i^N (q_i^{\pm})^\eta}{\eta - 1}, 
\end{equation}
where the index $\eta \in \mathbb{N}$ is a parameter that controls specific characteristics of the entropy function. Note that, for $\eta = 1$, $S^{\pm}_{1}$ coincides with the Shannon entropy implying that our computing paradigm is general enough to be applicable to different physical systems. In the Methods section~\ref{subsec:Maximum entropy simulation}, we show that the maximum-entropy distribution takes the form
$p_i^{\pm} = \frac{1}{\gamma^{\pm}}h(x_i \pm y_i - z^{\pm})$, $q_i^{\pm} = \frac{1}{\gamma^{\pm}} h( -x_i \mp y_i - z^{\pm})$, where $h(.)$ is a positive monotonically increasing convex function and $\gamma^{\pm} > 0$ are hyperparameters which are functions of $\eta$ and $U^{\pm}$~\cite{Wada2005Nonselfreferential}. The macroscopic ensemble potentials $z^{\pm}$ ensures that the constraints
\begin{subequations} \label{eqn:general MP}
\begin{align}
\sum_{i=1}^N h(x_i + y_i - z^+) + h(-x_i - y_i - z^+) &= \gamma^+, \\
\sum_{i=1}^N h(x_i - y_i - z^-) + h(-x_i + y_i - z^-) &= \gamma^-,
\end{align}
\end{subequations}
are satisfied under equilibrium (or steady-state) conditions. Fig.~\ref{fig:entropy_fig1}f conceptually illustrates the evolution of the probabilities to the maximum-entropy distribution corresponding to the steady-state condition~\ref{eqn:general MP}. In Fig.~\ref{fig:entropy_fig1}c, we show the macroscopic ensemble potentials $z^{\pm}$ which corresponds to the steady-state levels (like Fermi-level) above which the occupancy probability of the mesostate is zero. 

Equations~\ref{eqn:general MP} correspond to the generalization of a Margin Propagation analog computing paradigm~\cite{Ming_TCAS2012} which has previously been used to implement inner products~\cite{Ming_TCAS2012,kumar_tcas2024} and correlators~\cite{Kareem_JSSC2024,Aswin_ISSCC2025,Zhili_TCAS2024}. In~\cite{Zhili_TCAS2024}, it was shown that under specific statistical assumptions on ${\bf x}$ and ${\bf y}$, the expected value $\mathcal{E}(z^+ - z^-) \approx G({\bf x}^T{\bf y})$ is a function of the inner product or correlation ${\bf x}^T{\bf y}$. While this result was obtained when the steady state equations~\ref{eqn:general MP} were satisfied, we claim in this work that the computing framework also holds for far-from-equilibrium conditions as well. The framework is similar to the stochastic thermodynamics~\cite{peliti2021stochastic}, but applied here to generalized entropy and demonstrated for a practical analog computing hardware.  Fig.~\ref{fig:entropy_fig1}g shows the time evolution of $\mathcal{E}(z^+ - z^-)$ for the entropy function $S^{\pm}_{2}$ which corresponds to $h(x) = \text{max}(0,x)$ in~\ref{eqn:general MP}. The  Methods section~\ref{subsec:Signal Processing Gain & calibration} and supplementary materials describe the equations that have been used to simulate the time evolution of the system and the generation of the random vectors ${\bf x}$ and ${\bf y}$ with a predetermined correlation $R_{\infty}$. As shown in  Fig.~\ref{fig:entropy_fig1}g, for a given $R_{\infty}$, $\mathcal{E}(z^+ - z^-)$ follows a distinct temporal trajectory from the non-equilibrium state to the equilibrium state. This implies that non-equilibrium measurement of $\mathcal{E}(z^+ - z^-)$ could be used for computing correlations (or inner products).  Fig.~\ref{fig:entropy_fig1}h plots $\mathcal{E}(z^+ - z^-)$ as a function of $R_{\infty}$ and at different measurement times. Note that the function $G(R_{\infty})$ is monotonic, which means that it can be inverted to obtain an estimate of the correlation $R_{MP}$. More details on the post-calibration procedure for estimating function $G^{-1}$ to recover $R_{MP}=G^{-1} (z^+-z^-)$ is provided in section~\ref{subsec:Measurement_Setup} and the supplementary. In fact, as the ensemble size $N$ increases, the precision $R_{MP}$ in approximating the true correlation $R$ increases with similar asymptotic error rates as multiply-accumulate (MAC) based correlators according to $|\epsilon_{MP} |\approx |\epsilon_{MAC}|\sim \mathcal{O}\left(1/\sqrt{N}\right)$. This is shown in Fig.~\ref{fig:entropy_fig1}i for an MP-correlator with $h(x) = \text{max}(0,x)$ at different time steps, where the accuracy of correlation estimations, quantified by the signal processing gain (SPG) = $1/\epsilon_{MP}^2$, is higher than the MAC-based computation. The implication of the non-equilibrium computing of correlation (or inner products) and precision scaling with ensemble size is that the proposed approach can be used to push the limits of computational throughput and energy-efficiency. We demonstrate this using a proof-of-concept integrated circuit and hardware implementation shown in Fig.~\ref{fig:entropy_fig1}b.

\begin{figure}[htbp]
  \centering
\includegraphics[width=1.0\textwidth]{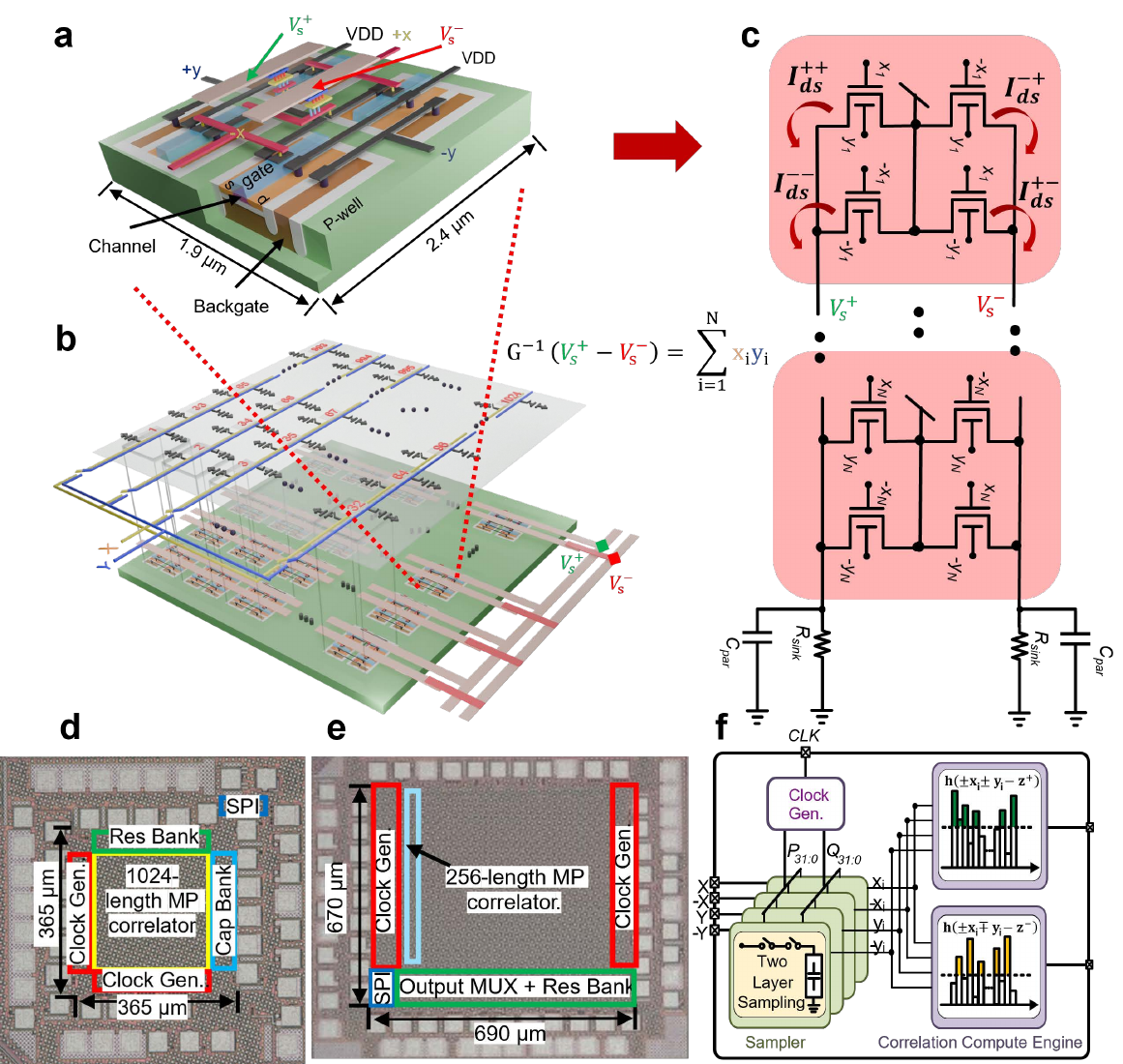}
  \caption{\textbf{Non-equilibrium MP-based analog computing realized in a Standard CMOS Process}. 
  (a) Layout and cross-sectional view of a 4-transistor MP-unit cell arranged in a centroid configuration. (b) Array of the MP-unit cells with interconnects and sampling capacitors. (c) Equivalent circuit model of the array comprising of dual-gate NMOS transistors enforcing rectangular symmetry constraints between the operands, and readout equivalent circuit. Die micrograph of two MP-ensembles fabricated in a 22\,nm SOI CMOS process and functioning as (d) N=1024 length correlator and (e) N=256 length correlator. (f)  Architecture showing samplers and clock generation modules interfacing with the correlator core in (d) and (e).
  }\label{fig:CMOS_schematic}
\end{figure}
While charge-based representation would be a natural mechanism to implement the ensemble dynamics shown in Fig.~\ref{fig:entropy_fig1}f on practical semiconductor devices, we instead use a currents to demonstrate the maximum-entropy computing paradigm. This choice is based on the fact that it is easier to track currents and associated leakage paths, as opposed to charges.
The energy barriers and the respective coupling between the potential wells in Fig.~\ref{fig:entropy_fig1}c were implemented using an equivalent circuit shown in  Fig.~\ref{fig:CMOS_schematic}a. The nonlinearity $h(.)$ in the MP-equation~\ref{eqn:general MP} is implemented by the drain-to-source current $I_{DS}(V_g,V_s,V_b)$ flowing through a metal-oxide-semiconductor (MOS) transistor and is a function of its gate, source, and bulk voltages. The MP-constraints and the parameter $\gamma^{\pm}$ in equation~\ref{eqn:general MP} are implemented using a load $R_{sink}$ at the source terminal. The rectangular symmetry constraint between the input operands in implemented applying the input $x$ to the transistor gate and the input $y$ to the back-gate (bulk) of the silicon-on-insulator (triple-well) transistor shown in Fig.~\ref{fig:CMOS_schematic}b. This realizes the operand summation $(x_i+y_i)$ implicitly in the transistor channel (corresponding to the height of the potential-well). Importantly, all four operands required by the MP compute are realized with four transistors driven with differential $\pm x_i$ and $\pm y_i$, resulting in a compact unit cell of $1.9\,\mu m \times 2.4\,\mu m$. The behavioral model for the circuit Fig.~\ref{fig:CMOS_schematic}a can be described using the differential equation
\begin{equation} \label{eqn:hardware_dynamics}
\sum_{i=1}^{N} I_{DS} (V_o+x_i,V_o\pm y_i,V_s^{\pm}(t) )+I_{DS} (V_o\mp x_i,V_o-y_i,V_s^{\pm}(t))= \frac{V_s^{\pm}(t)}{R_{\text{sink}}}+C_{par}\frac{dV_s^{\pm}(t)}{dt}
\end{equation}
where $I_{DS} (V_G,V_B,V_s)$ is the drain-to-source current of the nMOS transistor, $V_o$ is an offset voltage at the input and $C_{par}$ models the total load capacitance due to the interconnect and the read-out circuits. 
$R_{\text{sink}}$ determines the steady-state component of the normalization factor $\gamma$, and $C_{par}$ determines its transient component that determines the dynamics shown in Fig.~\ref{fig:entropy_fig1}g and h. The differential output, $V_{out}(t)=V_s^+(t)-V_s^-(t)$ is a monotonic function of correlation between ${\bf x}$ and ${\bf y}$ as $V_{out} \approx G_{MP}({\bf x}^T{\bf y})$, as shown in Section~\ref{subsec: MP characterization results}. Note that the equation~\ref{eqn:hardware_dynamics} does not assume a specific bias condition on the transistors (weak, moderate or strong inversion) and hence by construction is more versatile than other current-mode analog computing approaches. In the Supplementary section, we show the effect of interconnect inductance changes the transient response, and $V_{out}(t)$ still remains monotonic with respect to the correlation. However, the measured results in Section~\ref{subsec: MP characterization results} show that the effect of line inductance at the operating frequency is negligible compared to the line and load capacitance.


To validate the proposed dynamical MP-computing paradigm, we fabricated MP-ensembles of sizes 256 and 1024 respectively in a 22\,nm SOI CMOS process. The die micrograph for these prototypes are shown in Fig.~\ref{fig:CMOS_schematic}d-e. Since RF-edge computing was chosen as the specific use-case, both ensembles were used for computing correlations of lengths 256 and 1024 respectively. Also, both prototypes follow a common and a scalable architecture shown in Fig.~\ref{fig:CMOS_schematic}f. The SPI and clock generation are used to control the samplers. The differential architecture in Fig.~\ref{fig:CMOS_schematic}f first samples the input RF signals ($x,y$), which are then processed by the differential MP-correlator engine. In the sampling phase, the wideband RF inputs $x(t)$ and $y(t)$ are sampled using 2-level differential samplers, and the analog samples $\pm x_i$ and $\pm y_i$ $\forall$ $i= 1$ to $1024$ are stored on 20\,fF sampling capacitors that are connected to the MP-unit cells. The sampling capacitors are co-located with the MP-unit cells to reduce the parasitic effects and signal routing (see Fig.~\ref{fig:CMOS_schematic}b). The unit area including the sampling capacitors is $5.5\,\mu m \times 6.8\,\mu m$. Finally, a high-speed compute and process invariant biasing scheme is realized by connecting the source outputs with two current sinks, implemented as tunable resistors $R_{\text{sink}}$. During the compute phase, the $R_{\text{sink}}$ of the MP-correlator is enabled and the output is read as a differential voltage between the left and right arms. Finally, the differential output is mapped to correlation output as $R_{MP}=G_{MP}^{-1} (V_s^{+}-V_s^{-})$.

\section{Results}\label{sec2}
\subsection {Measured MP-ensemble Dynamics}
\label{subsec: MP characterization results}
Fig.~\ref{fig:timing diagram}a shows the timing diagram with separate phases when the input signals are sampled and then correlation is computed. Note that even though the compute phase depicts both the transient region and the steady-state region, the correlation can also be estimated during the sampling (due to device leakage). Conventionally, the outputs of computations are sensed during the quasi-steady-state, resulting in a total compute time $t >$100\,ns and a compute core efficiency of $<$1000\,TOPS/W.
While a closed form expression for~\ref{eqn:hardware_dynamics} is not derived here, numerical studies described in the  Supplementary section show that the differential output, $V_{out}(t)=V_s^+(t)-V_s^-(t) \propto R_{MP}(1-e^{-t/\tau_{out}})$, is proportional to the true correlation between the input signals and evolves as an exponential with a system time constant, $\tau_{out}$. The associated system time constant, $\tau_{out}$, is directly proportional to the aggregate parasitic capacitance at the output node that increases with the correlator length, and inversely proportional to the current consumption of the MP core. 

\begin{figure}[htbp]
  \centering
\includegraphics[width=1.0\textwidth]
{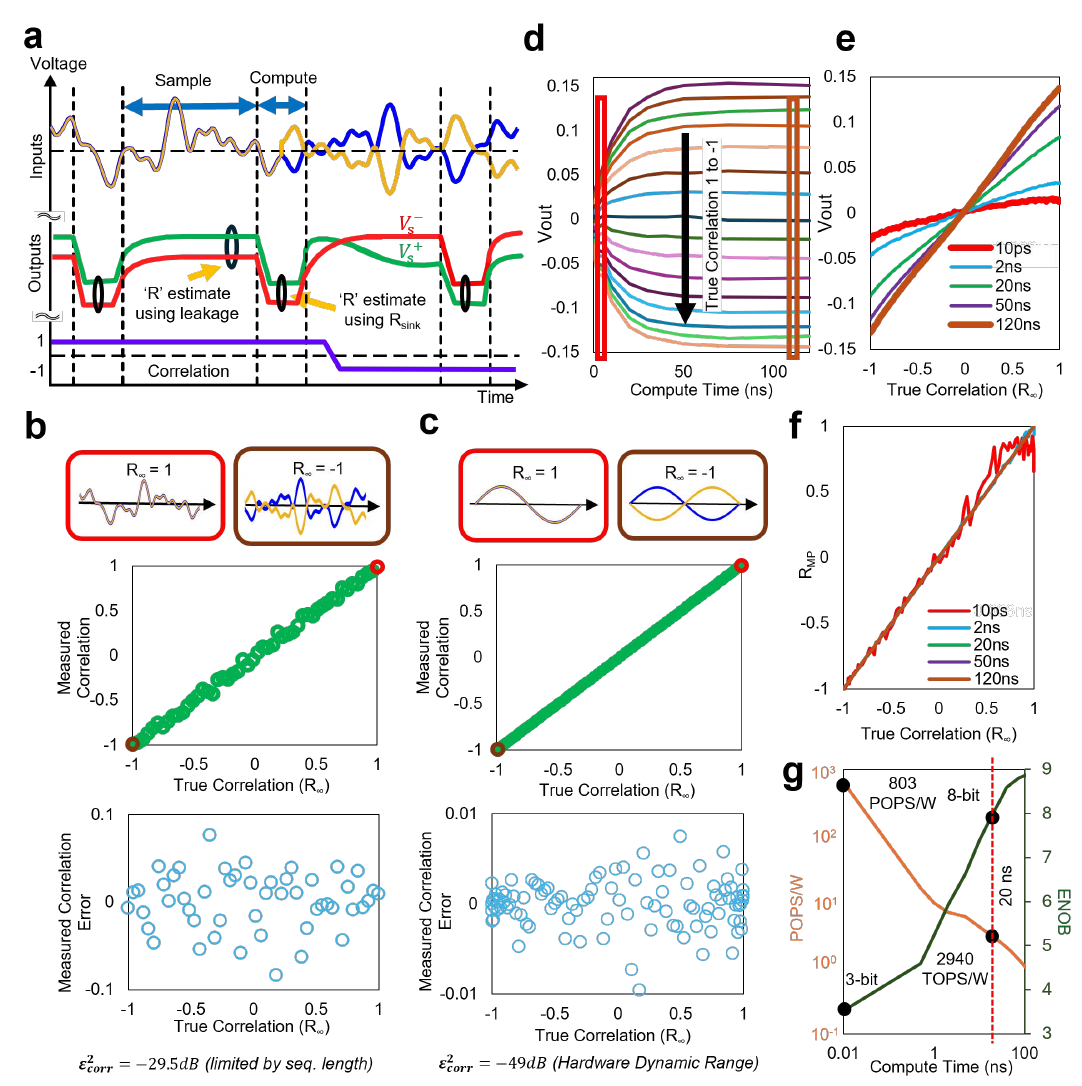}
  \caption{(a) Timing diagram illustrating MP correlator's operation, where sampling phase is followed by a phase to compute correlations. Correlation can also be estimated during the sampling phase, though with higher error due to a smaller number of sampled values.
  (b) Measured correlation error versus varying correlation between two random input vectors. The RMS error is –29.5 dB, closely approaching the theoretical limit of 
$1/\sqrt{1024}$.
(c) Measured correlation and corresponding error for periodic input signals. The use of deterministic inputs removes statistical uncertainty, isolating hardware-induced errors and resulting in an RMS error of –49 dB—comparable to an 8-bit digital multiplier. 
(d) Transient correlation response during the compute phase as the correlation between the two inputs is varied from -1 to +1.
(e) Output of the MP-correlator across varying compute time instances.
(f) Predicted correlation ($R_{MP}$) versus true correlation ($R_{\infty}$) across compute time after the compute-time dependent one-to-one mapping function.
(g) Estimated POPS/W and Effective Number of Bits (ENOB) across varying compute time, highlighting the energy efficiency and precision trade-off.
}\label{fig:timing diagram}
\end{figure}

This is also validated by hardware measurements. We characterized the error profiles of the implemented MP correlator at its steady state using pseudo-random inputs and periodic inputs across varying input correlations with the length-1024 correlator. The measured correlation between the input signals and estimation errors at compute time $t$=128\,ns are presented in Fig.~\ref{fig:timing diagram}b and \ref{fig:timing diagram}c. It is shown that for both inputs, the MP correlator generates monotonic responses to input correlations. The measured signal processing gain is 29.5\, and the effective compute precision is $>$8 bits. 

Fig.~\ref{fig:timing diagram}d-g depicts the time-domain evolution of the differential output across varying true correlation of periodic inputs for the length-256 MP correlator, and we measured the $\tau_{out}$ to be 20\,ns. As demonstrated through measured results in Fig.~\ref{fig:timing diagram}d-f and theoretical analysis in the supplementary material, it is not necessary to wait for complete steady-state convergence, and an accurate correlation prediction can be made by sensing the MP output during the transient phase. A reduction in compute time—by exploiting the transient phase sensing, translates directly to reduced energy consumption per correlation computation, given that energy scales with the product of power and compute duration. Nevertheless, this energy-efficiency gain introduces a trade-off: as compute time decreases, the output voltage amplitude diminishes, thereby degrading the effective compute precision (see Fig.~\ref{fig:timing diagram}g). Operating with a compute time equal to the system time-constant (20\,ns) resulted in high compute precision of 8-bit while achieving compute efficiency of 1793\,TOPS/W at the system level and
2940\,TOPS/W in the 256-length MP compute core. Notably, the 256-length system can operate with computation time of 10\,ps, achieving 0.803 Exascale Operations per Second per Watt (ExaOPS/W) with $>$3-bit precision. The method for estimating the computing power and energy is described in the Methods section~\ref{subsec:power computation}.

\begin{figure}[htbp]
  \centering
\includegraphics[width=1.0\textwidth]{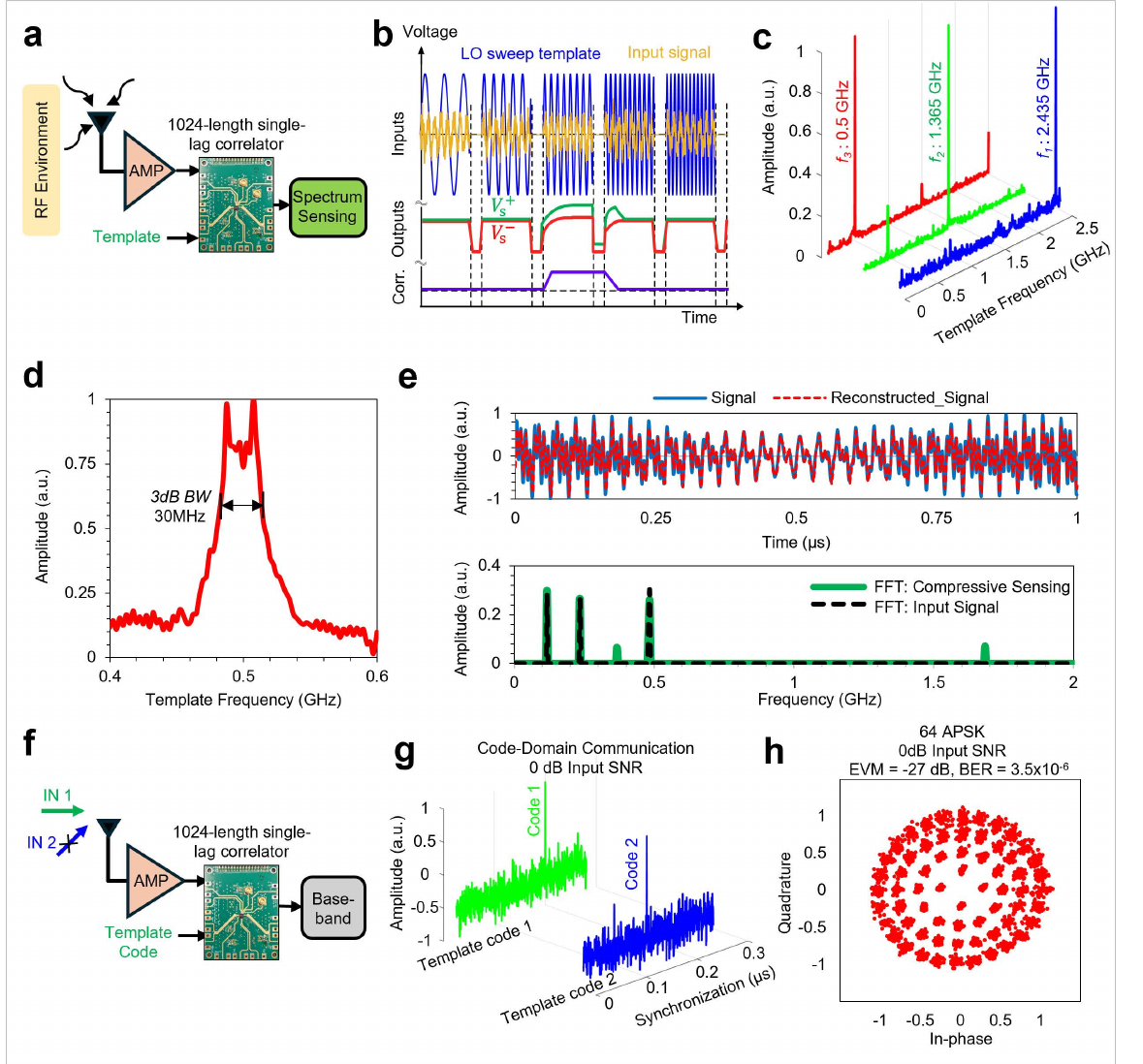}
  \caption{\textbf{Demonstration and evaluation of RF correlator IC in spectrum sensing and code-domain communication tasks.}
(a) Experimental setup for spectrum sensing using the 1024-length single-lag correlator.
(b) Timing diagram illustrating the operation of the spectrum sensing system.
(c) Measured correlation output demonstrating detection of input signal frequency.
(d) Measured output revealing the occupied bandwidth of the input signal.
(e) Compressive sensing-based reconstruction: comparison between the original signal, reconstructed waveform, and its frequency spectrum.
(f) Experimental setup for code-domain communication using pseudo-random spreading codes.
(g) Measured correlation output demonstrating selectivity to matched spreading codes.
(h) System performance in a 64-APSK (Amplitude and Phase Shift Keying) modulation scheme, achieving an error vector magnitude (EVM) of –27 dB which translates to $3.5\times10^{-6}$ BER.}
  \label{fig4}
\end{figure}
\subsection{Direct-RF Spectrum Sensing}

Joint communication and sensing is a key component of the emerging 6G wireless infrastructure~\cite{Shi_CommEngg2024,Chen_NatComm2025}. Continuous, low-power spectrum sensing is essential for efficiently utilizing the spectrum by detecting unused spectrum holes across wide frequency bands. These spectrum gaps can be opportunistically accessed by secondary users without compromising the fidelity of primary user communication. Existing spectrum sensors are primarily based on energy detection, which offers low detection accuracy, or cyclo-stationary feature extraction, which provides better sensitivity but at the cost of significant power consumption~\cite{Haykin_ProcIEE2009,Entesari_MicroMag2019}. In contrast, MP-correlator-based spectrum sensing offers a unique advantage by achieving high-fidelity detection while maintaining low power consumption.

The MP-correlator functions as a wideband spectrum sensor as demonstrated in Fig.~\ref{fig4}. When operating at a sampling rate of 5\,GS/s, it can scan the entire spectrum from DC to 2.5\,GHz with a resolution of 2.5\,MHz, enabling the frequency detection of an unknown input. In this experiment, the correlator input $X$ represents the sensor input, while input $Y$ serves as the template— a sinusoid with a varying frequency (see Fig.~\ref{fig4}b). As the template frequency is swept from DC to 2.5\,GHz, the correlator produces a high correlation at frequency bins corresponding to the input frequency. In Fig.~\ref{fig4}c, the MP-correlator identified three different input signal frequencies at 0.5\,GHz, 1.365\,GHz and 2.435\,GHz, respectively. Fig.~\ref{fig4}d demonstrates the spectrum sensing capability for a modulated input signal, where both the signal's bandwidth and center frequency can be accurately estimated using the MP-correlator based spectrum sensor.

In a scenario where the input signal is spectrally sparse, the total number of required correlations can be reduced by leveraging compressive sensing techniques~\cite{Arianada_TSP2012}. In this methodology, the spectrally sparse input is correlated with K=128 pseudo-random templates, where K is significantly smaller than the total number of 1024 spectral bins. The correlation outputs from these templates are processed with a CoSAMP recovery argorithm~\cite{Needell_ACM2010}, to estimate the spectral occupancy of the input signal. As depicted in Fig.~\ref{fig4}e, the predicted time domain signal and the frequency spectrum align closely with the ground truth, representing a strong use case of the MP-correlators in compressive sensing-based RF spectrum detection applications. 

\subsection{Code-Modulated Spread-Spectrum Communication}
\label{sec:xcorr}
Spread-spectrum communication used for low-SNR communications~\cite{Pickholtz_TComm1982} is another scenario where MP correlators can find applications (see Fig.~\ref{fig4}f), including: (i) code-synchronization, and (ii) code despreading and data reception. Fig.~\ref{fig4}g describes the capability of the MP-correlator to synchronize with a 0dB SNR, input signal that is code-modulated using a 5 Giga-chip/sec, pseudo-random code while rejecting the signal from the unmatched code. In Fig.~\ref{fig4}h, a code despreading application is demonstrated by employing the MP correlator in a coherent 64-APSK signal down-conversion and demodulation. A 64-APSK, 5Giga-chip/sec code-modulated, 2\,GHz carrier signal with low SNR of 0\,dB, is demodulated by correlating with the in-phase (I) and quadrature (Q) carrier templates. An error vector magnitude (EVM) of -27\,dB is measured matching the theoretical estimates when using a 1024-length code, and this translates to a Bit Error Rate (BER) lower than $3.5\times10^{-6}$. These measurements illustrate the capability of MP-correlators in code synchronization and signal demodulation.

\section{Discussions}\label{sec3}
In this paper, we proposed and demonstrated the feasibility of a maximum-entropy-based analog computing paradigm that can push the limits of energy efficiency and computational throughput. At the core of this approach is a margin propagation (MP) framework which can be viewed as a physical manifestation of the generalized maximum entropy principle and can be applied to different analog device ensembles and emerging compute-in-memory architectures. In this work the energy barriers formed by transistors were considered to be mesostates that are coupled together by interconnects and we showed that MP dynamics evolve towards a maximum-entropy configuration in the presence of rectangular symmetry and energy conservation constraints. The desired output which is correlation is a macroscopic property of the ensemble, just like pressure and temperature in kinetic theory of gasses~\cite{JaynesPhysRev1957}. This physics-based computation not only leads to high energy-efficiency and computational throughput, but the statistical approach for analog computing makes the paradigm robust to noise, environmental fluctuations, and device parasitics. Furthermore, we also show that the computing paradigm holds far-from-equilibrium or transient conditions.

In traditional analog computing architectures, transients are generally avoided due to: (a) artifacts introduced by parasitic elements and the mismatch; and (b) the physical computation that is implemented by the array during the transient phase is difficult to model. The proposed maximum-entropy and MP computing is robust to the choice of the non-linear function $h(.)$, as long as it satisfies very general properties~\cite{Zhili_TCAS2024} and can be implemented using different devices other than just transistors. Furthermore, the generic properties of $h(.)$ holds during the transients. For instance, during the transient state, the current flowing through the transistor is not only a function of the steady-state drain current, but is also a function of current arising due to the redistribution of transistor channel charge and the displacement current flowing through the interconnects and parasitics. The invariance of the property of $h(.)$ in steady and transient states is the main reason why the proposed approach can not only achieve high computational throughput and energy-efficiency but is also robust to parasitics. Another source of robustness of the proposed approach to device mismatch is the use of rectangular symmetry when applying differential inputs and during device layout, combining the outputs of the array. This type of symmetry is routinely exploited in designing current mirrors that are robust to fabrication artifacts and to PVT variations. In Supplementary material, we provide additional evidence of robustness using measured data. 

We measured the performance of MP-correlators in two different lengths (256 and 1024). The trade-off between energy-efficiency and computational precision with respect to the measurement time is shown in Fig.~\ref{fig:timing diagram}d-g for a 256-length correlator. The 22\,nm SOI CMOS implementation resulted in a system time constant $\tau_{out}$=20\,ns. When the compute time is lower than $\tau_{out}$, the low output swing (see Fig.~\ref{fig:timing diagram}e) results in compute precision $<$8-bit ENOB (see Fig.~\ref{fig:timing diagram}f and g). A compute time of 10\,ps achieves $>$3-bit ENOB and 803\,POPS/W efficiency. However, as can be seen in Fig.~\ref{fig:timing diagram}e-f, for the 10\,ps measurement, the effective precision of the correlator (post-calibration) is asymmetric about the origin, with a higher effective precision when the input correlation is within -1 to 0. This asymmetry could be attributed to the device layout, mismatch, and parasitics, and raises the possibility that the precision could be further improved in future generations of this design. However, for the current implementation, a compute time of 20\,ns results provides an optimal trade-off with 8-bit ENOB and 2.9 POPS/W efficiency.

\begin{figure}[H]
  \centering
\includegraphics[width=1.0\textwidth]{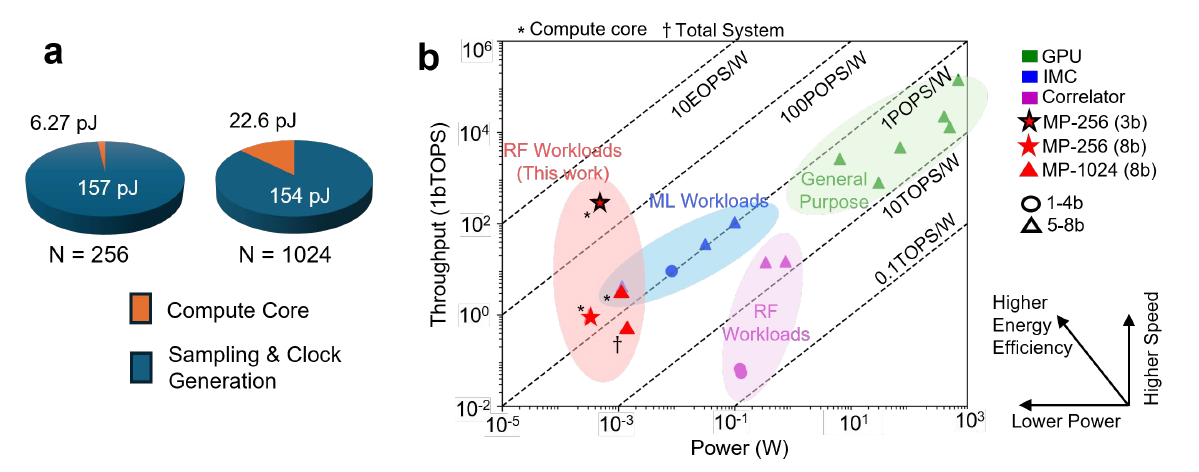}
  \caption{\textbf{Performance of Direct RF-Sampling MP Analog Correlator}. 
 (a) Pie chart of energy consumption per frame at different sampling rates for the 256-length with compute time of \SI{20}{\nano\second}, \SI{2800}{\ohm} of $R_\text{sink}$, \SI{0.5}{\volt} supply and the 1024-length MP correlator with compute time of \SI{20}{\nano\second}, \SI{400}{\ohm} of $R_\text{sink}$, \SI{0.8}{\volt} supply to compute core. Both samplers have \SI{0.8}{\volt} power supply.
 (b) Comparison with other inner-product and correlator hardware platforms. The proposed dynamic MP approach results in superior compute core efficiency, exceeding 800\,POPS/W in 3-bit precision and 2900\,TOPS/W in 8-bit precision for length 256 at 4GS/s, and 2885\,TOPS/W compute core efficiency, 370\,TOPS/W at the system level in 8-bit precision for length 1024, while directly sampling the RF data at 10GS/s. 
  }\label{fig:performance}
\end{figure}

The omission of the high-speed digital compute and analog-to-digital conversion of the proposed approach has resulted in significant energy savings as depicted in Fig.~\ref{fig:performance}a for the MP-correlators. This correlator IC demonstrates high computing efficiencies across a range of sampling rates. For the 256 length MP correlator, the power consumption during the sampling phase ranges from 1.2 to 4.9\,mW at sampling rates between 1 and 4\,GS/s. The compute phase consumes \SI{6.27}{\pico\joule} for compute time of 20\,ns with 8-bit ENOB, resulting in a compute-core efficiency of 2940\,TOPS/W, as shown in Fig.~\ref{fig:timing diagram}g. As for the 1024-length MP correlator, the IC consumes 0.16–1.5\,mW for sampling rates ranging from 1 to 10\,GS/s during the sampling phase. In the compute phase—during which the sampler is disabled—the power consumption is 1.2\,mW. For a compute time of 20\,ns, each correlation operation performs the equivalent of 65.3K binary operations and consumes 0.18\,nJ of energy. This corresponds to a system-level compute efficiency of 370\,TOPS/W, and the standalone MP compute core compute efficiency is 2885\,TOPS/W. 
The sampler clock generation limits the system efficiency of the correlator, and the overall system efficiency can be improved by sharing the sampled voltages across multiple correlators, as reported in~\cite{Aswin_ISSCC2025}. 

In Fig.~\ref{fig:performance}b, we compare the energy-efficiency and computational throughput of MP-based correlators with digital platforms such as the NVIDIA H100, A100, Jetson, Tesla T4, AMD MI250, and Qualcomm Snapdragon 8 Gen 3 \cite{nvidia_H100,nvidia_a100,nvidia_jetson,nvidia_tesla_t4,amd_mi250,qualcomm_snapdragon8gen3}. These digital architectures offer high computational throughput but at the cost of lower energy-efficiency, reflecting the trade-off between the two performance metrics.  In-memory compute systems \cite{wang2022dimc,he2023_isscc,jia2021_jssc,wang2023_jssc} perform computations in memory units itself thus saving power and time on massive data transfers between memory and compute processors thus achieving higher computing efficiencies compared to conventional Von-Neumann digital architectures. While these parallel in-memory compute architectures are ideal for audio/video processing and recognition systems, they are limited by low IO speeds, making them not suitable for high-speed RF applications. MP-correlators on the other hand can achieve high computational throughput (including high I/O speeds) while maintaining better compute efficiencies. More information about this is provided in the supplementary section.

When compared to other correlator archtiectures, shown in Fig.~\ref{fig:performance}b, previous analog and RF-domain correlators \cite{Kareem_JSSC2024,javed2016linear,Tang_IMS2019,Giannino_ISSCC2019,Guermandi_JSSC2017} suffer from low throughput and relatively high power consumption, often constrained by limited scalability and analog precision. Prior MP-based correlators are implemented in the charge domain~\cite{Kareem_JSSC2024}, while efficient than integrator-based analog correlators~\cite{Tang_IMS2019,Tang_TCAS2021}, the implementation of the operand generation through capacitor stacking limits the unit cell area to 220\,$\mu m^2$ and limited the compute efficiency of the MP compute core to 150\,TOPS/W. The proposed concept with inherent operand generation overcomes these limitations and thus results in high-speed analog cross-correlators over several GHz bandwidth, achieving orders of magnitude improvements in both power and area efficiency. 

The demonstrated near ExaOPS-per-watt energy-efficiency metric is equivalent to dissipating approximately 100kT of energy per bit. Even though this metric is three orders of magnitude larger than the Landauer energy dissipation limit~\cite{Landauer1961EntropyProduction}, it is in the same order as thermodynamic computing paradigms based on stochastic dynamics~\cite{MelansonNatureCommThermoAIComputing2025}, DNA or other molecules~\cite{KempesBioComputingEfficiency2017,KonopikFiniteTimeParallelCompute2023} and adiabatic computers~\cite{FrankReverseComputing2005, GertAdiabatic2007, BarendsAdiabaticQuantumNature2016}. However, compared to these thermodynamic computing architectures, our approach can achieve computational throughput (operations per second) that is orders of magnitude higher. We attribute this advancement to a better matching between the physics of computation and the physics of the underlying hardware substrate. In this proof-of-concept implementation, we demonstrated near ExaOPS-per-watt efficiency at only three bits of precision. However, we believe that the precision can be significantly improved by a better device layout that maintains the symmetry needed for MP-based computation and by using a more precise read-out mechanism. This potential can be seen from the measurements in Fig.~\ref{fig:timing diagram}d-e, where for the read-out time of 10ps, the precision of the measured correlation is not symmetric. Another attribute of the proposed computing paradigm is that the output of the system is a non-linear but monotonic function $G(.)$ of the correlation (or inner-product). While in this work we calibrate using $G^{-1}(.)$ to estimate the correlation, the nonlinearity could directly be used as an activation function for neural network implementations. Exploiting and designing $G(.)$ for specific ML tasks will be the focus of future work.

In summary, this work not only introduces a new paradigm for high-speed analog computation but also sets the stage for future developments in ultra-low-power, high-performance signal processing for applications at the RF edge. By leveraging the statistical and collective nature of analog domain computing and the MP paradigm, this approach can unlock new possibilities for the deployment of low-cost, high-efficiency devices in areas ranging from next-generation wireless networks to advanced sensing technologies. Interesting future research could include scaling the MP correlator design to even larger correlation lengths to increase processing gain, reduce the unit cell size, and explore its potential in other domains such as optical communications, LiDARs, and matrix-vector multipliers targeting machine learning and AI hardware.

\section{Methods}\label{sec4}

\subsection{MP-computing as a Maximum-entropy configuration}\label{subsec:Maximum entropy simulation}
The maximization of entropy $S_{\eta}^+$ and $S_{\eta}^-$ in~\ref{eqn:q-functional} with the normalization of the probability distribution and the energy ensemble constraints in~\ref{eqn:ensemble constraint} and with Lagrange multipliers $\alpha$ and $\beta$ respectively yields
\begin{subequations}
\begin{align}
    \frac{d}{dp_i^{\pm}}\left(S_{\eta}^{\pm} + \alpha^{\pm}\sum{(p_i^{\pm} + q_i^{\pm})}+\beta^{\pm}\sum{(x_i\pm y_i)(p_i^{\pm} - q_i^{\pm})}\right) = 0, \\
    \frac{d}{dq_i^{\pm}}\left(S_{\eta}^{\pm} + \alpha^{\pm}\sum{(p_i^{\pm} + q_i^{\pm})}+\beta^{\pm}\sum{(x_i\pm y_i)(p_i^{\pm} - q_i^{\pm})}\right) = 0,
\end{align}
\end{subequations}
which gives follows ~\cite{Wada2005Nonselfreferential}
\begin{subequations} \label{eqn:steady state entropy}
\begin{align} 
         p_i^+ =  \left[\frac{(\eta-1)(\alpha^+ + \beta^+(x_i + y_i))}{\eta} \right]_+^{\frac{1}{\eta-1}}, \\
         q_i^+ =  \left[ \frac{(\eta-1)(\alpha^+ - \beta^+(x_i + y_i))}{\eta} \right]_+^{\frac{1}{\eta-1}},\\
        p_i^- =  \left[\frac{(\eta-1)(\alpha^- + \beta^-(x_i - y_i))}{\eta} \right]_+^{\frac{1}{\eta-1}}, \\
         q_i^- =  \left[ \frac{(\eta-1)(\alpha^- - \beta^-(x_i - y_i))}{\eta} \right]_+^{\frac{1}{\eta-1}},
\end{align}        
\end{subequations}
where $[\cdot]_+ = \text{max}(0,.)$ ensures the positivity of the distributions. For $1 <\eta \le 2$, the distribution is a non-linear, monotonically increasing, convex function with respect to the input, which yields the formulation of a generalized MP computing by rewriting the normalization constraint as below, \begin{subequations}\label{eqn:fromqtoh}
    \begin{align}
\sum_i\left[\alpha^+/\beta^+   + (x_i+y_i) \right]_+^{\frac{1}{\eta-1}} + \left[\alpha^+/\beta^+ -  (x_i+y_i) \right]_+^{\frac{1}{\eta-1}} = \gamma, \\
        \sum_i\left[\alpha^-/\beta^-   + (x_i-y_i) \right]_+^{\frac{1}{\eta-1}} + \left[\alpha^-/\beta^- -  (x_i-y_i) \right]_+^{\frac{1}{\eta-1}} = \gamma.
    \end{align}
\end{subequations}
The form in equation~\ref{eqn:fromqtoh} is equivalent to the equation~\ref{eqn:general MP}. To determine the equivalent normalization factor $\gamma$ and output thresholds $z^+$ and $z-$ in~\ref{eqn:general MP}, we get $\gamma:= ((\eta-1)|\beta|/\eta)^{\frac{1}{1-\eta}}$ from~\ref{eqn:fromqtoh}, and the output thresholds $z^+$ and $z^-$ can be subsequently determined as $-\alpha/\beta$.

\subsection{Signal Processing Gain vs Correlation length} \label{subsec:Signal Processing Gain & calibration}
Fig.~\ref{fig:entropy_fig1}h presents the simulated signal processing gain (SPG) in decibels (dB) as a function of correlation length for both the MAC-based correlation and MP-based correlation approaches at different time steps in its dynamics ($t=2$, 10, and 80). The MAC-based correlator employs full-precision inner product computations. The MP-based correlator replaces the nonlinear function h(.) in Eqn.~\ref{eqn:general MP} with a linear-rectifying function, effectively implementing the Tsallis entropy configuration with $\eta = 2$. Although both architectures produce outputs that exhibit monotonic behavior with respect to the true input correlation, accurate recovery of the correlation values from the MP correlator requires a post-calibration step. In this calibration, a fifth-order polynomial regression model is trained on 500 Gaussian distributed pseudo-random input vectors and evaluated on an additional 1,000 test cases to establish a one-to-one mapping from the correlator's output to the true correlation, and plotted values are the mean of SPG values obtained for test cases. Further details regarding the numerical simulation for the MP dynamics and the calibration methodology are provided in the Supplementary material.

Both the MAC and MP approaches demonstrate a 3 dB improvement in SPG for each doubling of correlation length, consistent with theoretical expectations. These results confirm that the MP-based correlator, despite its approximate nature, can reliably estimate correlation values with better accuracy compared to full-precision methods.

\subsection{Estimation of Power and Energy Dissipation}\label{subsec:power computation}
Since the DC supply can only record average current consumption $I_{static}$ after the correlator output reaches a steady state, we resorted to a conservative approach to estimate the transient power consumption based on the circuit model in~\ref{eqn:hardware_dynamics}. This has also been validated by the numerical and behavioral simulations shown in the Supplementary section. From the measurements, the time constant $\tau = R_{sink}C_{par}$ can be estimated to compute the parasitic capacitance $C_{par}$. Then, the transient current can be computed as follow
\begin{equation}
    I(t) = I_{\text{static}} + C_{par}\frac{dV_{out}}{dt}.
\end{equation}
Note that this $I(t)$ is a conservative estimate and upper bounds the actual current for two reasons: First, the $C_{par}$ estimated here includes the parasitic capacitance of the measuring instrument; Second, $I_{\text{static}}$ is larger than the actual current going through the resistor path, because the steady-state $Vout$ is larger than transient $V_{out}(t)$, and $I_{\text{static}}$ also includes currents through other leakage paths. From $I(t)$, we can compute the transient power for the compute core by $P(t) = I(t)V_{dd}$, and then integrate the transient power over the compute time duration to estimate the energy consumption by the compute core.

\subsection{1024-length correlator Implementation} 

The RF correlator consists of two primary blocks: (1) a sampler that sequentially stores input samples onto 1024 capacitors, and (2) a correlation compute engine that estimates the correlation between the input signals. To minimize parasitic loading on the RF input node, the 1024-phase sampler adopts a two-layer architecture. In the first layer, the RF input is sampled by 32 switches driven by 32-phase non-overlapping clocks with a period of 32$T_s$ and pulsewidth equal to the sampling period, $T_s$. Each output from this first layer is further sampled in the second layer using 32 additional switches, controlled by another set of 32-phase non-overlapping clocks with a longer period of 1024$T_s$ and pulse width of 32$T_s$. This architecture results in an effective ON resistance of 600\,$\Omega$ with a 20\,fF sampling capacitor co-located with the MP-unit cells as shown in Fig.~\ref{fig:CMOS_schematic}c. This hierarchical clocking scheme enables power-efficient 1024-phase sampling using only 32 high-speed switches and 1024 low-speed switches, reducing both RF input loading and complexity in clock generation. The resulting 1024 samples of the input signals, $x(t)$ and $y(t)$, are routed to 1024 MP unit cells designed using 500\,nm/70\,nm NMOS transistors with isolated backgates. The current sink is implemented using a tunable resistor from 500\,$\Omega$-2k\,$\Omega$. The differential output from the MP array is buffered using a 50\,$\Omega$ voltage buffer to enable voltage readout.

\subsection{Measurement Setup}
\label{subsec:Measurement_Setup}
The RF correlator ICs were evaluated under a variety of experimental scenarios to characterize their performance in both controlled and real-world conditions. A key metric of interest is the hardware dynamic range (HDR), which directly relates to the effective computational precision of the IC. In correlation-based computation, two primary sources of error must be considered: (1) finite sample length effects, $\epsilon_{len}$, and (2) intrinsic hardware computation errors $\epsilon_{HDR}$. When correlating pure sinusoidal signals, the contribution from finite-length effects becomes negligible, thereby isolating errors attributable solely to hardware non-idealities.

For correlator characterization, two inputs were generated using a Tabor Electronics Proteous Arbitrary Waveform Generator (Arb), P9484D, while a Keysight analog signal generator, E8257, provided the local oscillator (LO) for sampling clock generation. A correlation measurement consists of two primary sources of error: (1) finite sample length effects, $\epsilon_{len}\approx 1/\sqrt{N}$ with $\epsilon_{len}\rightarrow 0$ for larger sequences, and (2) intrinsic hardware computation errors that include MP-approximation and noise/mismatch, $\epsilon_{HW}$. For real-world scenarios, it is required to have $\epsilon_{HW}<<\epsilon_{len}$. The correlator was measured by injecting two analog vectors $X=S_1$ and $Y=R_{\infty} S_1+\sqrt{1-R_{\infty}^2}S_2$ with varying correlations $R_{\infty}$, where $S_1$ and $S_2$ are uncorrelated analog random vectors. The measured correlation is calculated from the output voltage using a pre-trained 5$^{th}$-order polynomial function, $R_{MP}=G^{-1}(V_{out})$. The signal processing gain ($SPG=1/\epsilon_{len}^2+1/\epsilon_{HW}^2$) of the correlators is calculated from the RMS error between the measured correlation $R_{MP}$ and the true correlation $R_{\infty}$ and has been shown to approach the theoretical limit (see Fig.~\ref{fig:timing diagram}b). The hardware dynamic range (a.k.a. compute precision) $=1/\epsilon_{HW}^2$, was measured by choosing the $S_1$ and $S_2$ as a cosine and a sinusoid, effectively making $\epsilon_{len}=0$ (see Fig.~\ref{fig:timing diagram}c). A detailed description of the measurement setup is presented in the supplementary material.


\section*{Data Availability}\label{sec_data}

The data supporting the figures within this paper
are available from the corresponding authors on a reasonable request.

\section*{Code Availability}\label{sec_code}

The MATLAB codes used in simulation/emulation studies on the CPU and FPGA platforms are available from the corresponding authors upon reasonable request.

\section*{Acknowledgments}\label{sec_ack}
 This work was supported by the National Science Foundation grants CNS - FuSE2 - 2425444, ECCS - 2513548, CNS - 2128535, ECCS - 2332166 and FET -2208770 and by the DARPA MAX program, award no. FA8650-23-2-7309. The authors would like to thank GlobalFoundries University MPW program for fabrication support and Tabor Electronics for equipment support. 

\section*{Author Contributions}\label{sec_contributions}
All authors were involved in the design review of the IC prototype. A.U and A.Nag. were responsible for the simulations and layout of the IC. S.C. and X.Z. formulated the maximum-entropy principle for MP-computing. X.Z. conducted the behavioral simulations. A.U. and A.Nag. tested and characterized the fabricated IC. All authors/co-authors contributed to proof-reading and writing of the manuscript.  

\section*{Competing interests}\label{sec_COI}
S.C., A.Nag, and A.Nat. are named as inventors on U.S. patents related with MP-correlators, and the rights to the intellectual property are jointly owned by Washington University in St. Louis and Oregon State University. 

\bibliography{sn-bibliography}
 




\newpage
\setcounter{page}{1}
\section*{Supplementary Material: Maximum-Entropy Analog Computing Approaching ExaOPS-per-Watt Energy-efficiency at the RF-Edge} 
\setcounter{section}{0}
\section{Mathematical framework for estimating correlations}\label{secA1}
The true correlation $R_{\infty}$ between two real-valued random variables $X \in \mathbb{R}$, $Y \in  \mathbb{R}$ is generally defined as the expectation of their product $\mathcal{E}[XY]$ as below
\begin{equation}
    R_{\infty} \mathrel{:=} \mathcal{E}[XY] =  \int_{-\infty}^{\infty} \int_{-\infty}^{\infty} xy \enskip p(X=x,Y=y)dxdy,
    \label{eqn_defcorr}
\end{equation}
In practical scenarios, the joint distribution $p$ is not known a priori, so the expectation values are approximated using the inner product of $N$ independent pairs of samples $\{x_i, y_i\}^N_{i=1}$ drawn from the distribution $p$, yielding the empirical correlation:
\begin{equation} 
\hat{R}_N =\frac{1}{N}\boldsymbol{x}^T\boldsymbol{y} = \frac{1}{N}\sum_{i=1}^{N} x_i y_i,
\label{eqn:empcorr}
\end{equation}
where $\boldsymbol{x} \in \mathbb{R}^\text{N}$ and $\boldsymbol{y} \in \mathbb{R}^\text{N}$ denotes the vector form of input samples. 
The empirical correlation converges to the true correlation 
$R_{\infty}$ as sample size $N$ grows large, assuming finite second-order moments or variance $\mathcal{E}[(XY)^2]$,
\begin{equation}
\text{Var}\left(\frac{1}{N}\sum_{n=1}^{N} x_n y_n - R_{\infty}\right) \le \epsilon \stackrel{N \rightarrow \infty}{\longrightarrow} 0.
\label{eqn_empconverge}
\end{equation}
Conventional correlators use analog and digital MAC circuits to compute $\hat{R}_N$ in~(\ref{eqn:empcorr}). 
Rather than relying on discrete MAC operations, in prior work~\cite{Zhili_TCAS2024}, it was proposed to use a family of nonlinear functions called margin propagation (MP) that naturally emerge from physical substrates to compute correlations. In this work, we propose an alternative collective analog computing approach by extending the MP framework to vector functions. This new paradigm interprets the correlation operation as the steady-state response of a statistical ensemble governed by maximum entropy principles and symmetry constraints. In this setting, the entire physical substrates, including the computing arrays, interconnects, and parasitics, are treated as a monolithic computing unit of nonlinear functions 
$f:\mathbb{R}^\text{N} \rightarrow \mathbb{R}$. As the system evolves towards a high entropy state shaped by inputs $\boldsymbol{x},\boldsymbol{y}$, the expected macroscopic output is a monotonic function $G_{MP}:\mathbb{R} \rightarrow \mathbb{R}$ of the true correlation,
\begin{equation}
\mathcal{E}[f(\boldsymbol{x},\boldsymbol{y})] = G_{MP}(R_{\infty}),
\end{equation}
When inputs are drawn from stationary distributions, $G_{MP}$ can be analytically characterized or calibrated. The inverse $G_{MP}^{-1}$ enables accurate correlation inference from nonlinear statistics, often outperforming MAC-based estimates at finite $N$. As $N$ grows large, the signal processing gain reduces estimation error like the empirical approach as well
\begin{equation}
\text{Var}\left(G_{MP}^{-1}(f(\boldsymbol{x},\boldsymbol{y})) - R_{\infty}\right) \le \epsilon \stackrel{N \rightarrow \infty}{\longrightarrow} 0.
\end{equation}

In section~\ref{subsec:h existance}, we first describe the properties of the MP function $\phi: \mathbb{R}^\text{N} \xrightarrow{} \mathbb{R}$. Then in Subsection~\ref{subsec:lipschtz and convex}, we demonstrate that $\phi$ is a convex function when applied to symmetric operands. Finally, in Subsection~\ref{subsec:MP correlator}, we use MP function $\phi$ with symmetric operands $\boldsymbol{x}\pm\boldsymbol{y}$ and $\pm\boldsymbol{x}\mp\boldsymbol{y}$ to construct the MP-based correlator function $f:\mathbb{R}^\text{N} \rightarrow \mathbb{R}$.
Leveraging Price's Theorem~\cite{Price1, PriceExtension, CommentOnExtension}, we show that the expected output of MP correlators $f(\boldsymbol{x},\boldsymbol{y})$ is a monotonic function $G_{MP}(R_{\infty})$ of the input correlation $R_{\infty}$, enabling accurate calibration and inference of true correlation. 

Without any loss of generality, we assume $X$ and $Y$ are zero-mean and unit-variance random variables. The correlation $R_{\infty} = \mathcal{E}[XY]$ is equivalent to Pearson's correlation coefficient~\cite{pearson1895note} $R$ in this case. When $X,Y$ have nonzero means $\mu_X$ and $\mu_Y$, we can recenter them to obtain correlations
\begin{equation}
    \mathcal{E}[XY] = \mathcal{E}[(X-\mu_X)(Y-\mu_Y)] + \mu_X \mu_Y.
\end{equation}

\section{Margin Propagation function and properties}\label{subsec:h existance}
Given an input vector $O = \{o_i\}, i = 1,..,N$, the output of an MP function  $z =MP(O,\gamma)$ is the solution to the following equation
 \begin{equation}
    \sum_{i=1}^N h(o_{i} - z) = \gamma,
    \label{eqn:MPbasic}
\end{equation}
where $h(.)$ is a monotonically increasing function, and $\gamma > 0$ is a hyper-parameter. It can be easily shown that for a given input vector $O$ of dimension $N$ and a fixed constraint $\gamma$, MP defines a function $\phi: \mathbb{R}^\text{N} \xrightarrow{} \mathbb{R}$, such that $\phi(O)$ is unique.
\begin{lem}\label{lem1} 
For a given constraint $\gamma$ and input vector $O$, the output of MP function $z = \phi(O)$ is unique.
\end{lem}
\begin{proof}
We prove this by contradiction. Suppose $\phi({O}) = z_1 $ and $z_2$, and $z_1 < z_2$, because function $\phi$ is increasing monotonically, 
 \begin{equation}
    \sum_{i=1}^M h(o_{i} - z_1) > \sum_{i=1}^M h(o_{i} - z_2),
\end{equation}
which means the two solutions cannot satisfy the constraint  $\gamma$ at the same time. Therefore, the function of $\phi({O})$ is well defined.
\end{proof}
\subsection{MP Functions with Symmetric Operands} \label{subsec:lipschtz and convex}
To construct an MP-based correlator, the monotonic function $h$ needs to be convex, and the operands $O = [o_i, -o_i]$ as the input vector are symmetrically constructed. This construction makes the MP function an L-Lipschitz continuous convex function.

The following lemma proves the convexity of symmetrically constructed MP functions. 
\begin{lem}\label{lem3} 
With a monotonically increasing convex function $h(.)$, symmetric operands $O = [o_i, -o_i]$ and constraint $\gamma$, the MP function $z = \phi(O)$ is convex.
\end{lem}
\begin{proof}
The theorem can be proved by demonstrating that the Hessian of the MP function $\nabla^2 \phi(O)$ is a positive semidefinite(PSD) matrix.

With symmetric operands $O = [o_i, -o_i]$, from Eqn.~(\ref{eqn:MPbasic}) we have
\begin{equation}
    \sum_{i=1}^{n} h(o_i - z) + h(-o_i - z) = \gamma.
\end{equation}
Taking the partial derivative with respect to $o_j$ and applying the chain rule,
 \begin{equation}
     \sum_{i=1}^{n} \frac{\partial h(o_i - z)}{\partial o_j} + \frac{\partial h(-o_i - z)}{\partial o_j} = 0,
 \end{equation}
 \begin{equation}
\sum_{i=1}^{n} \frac{\partial h(o_i - z)}{\partial (o_i - z)} \frac{\partial (o_i - z)}{\partial o_j} + 
\sum_{i=1}^{n} \frac{\partial h(-o_i - z)}{\partial (-o_i - z)} \frac{\partial (-o_i - z)}{\partial o_j} = 0. 
 \end{equation}
 Rearranging the terms, we arrive at
  \begin{equation}
\frac{\partial h(o_j - z)}{\partial (o_j - z)} -  \frac{\partial h(-o_j - z)}{\partial (-o_j - z)} = \frac{\partial z}{\partial o_j} \sum_{i=1}^{n} \left[ \frac{\partial h(o_i - z)}{\partial (o_i - z)} +  \frac{\partial h(-o_i - z)}{\partial (-o_i - z)} \right].
\end{equation}
So the first order derivative is expressed as
\begin{equation}
 \frac{\partial z}{\partial o_j}  = \frac{h'(o_j - z) - h'(-o_j - z)}{\sum_{i=1}^{n} \left[ h'(o_i - z) + h'(-o_i - z) \right]},
\end{equation}
where $h'(o - z)$ denotes
 $\frac{\partial h(o - z)}{\partial (o - z)}$. Note that the norm of this expression is bounded by 1 because $h'(o - z) > 0$ due to its increasing monotonicity. Hence, the gradient norm $\|\nabla \phi(O)\|$ is bounded as well.

Applying the above techniques again, we can get the second-order derivative as  
\begin{equation}
\frac{\partial^2 z}{\partial o_j^2} = \frac{d_j^{'+} \sum_{i=1}^{n} d_i^+ - d_j^{'-} d_j^- - d_j^- d_j^{'-} + d_j^- d_j^- \sum_{i=1}^{n} \frac{d_i^{'+}}{\sum_{i=1}^{n} d_i^+}}{\left( \sum_{i=1}^{n} d_i^+ \right)^2},
\end{equation}

\begin{equation}
\frac{\partial^2 z}{\partial o_j \partial o_k} = \frac{- d_j^{'-} d_k^- - d_k^{'-} d_j^- + d_k^- d_j^- \sum_{i=1}^{n} \frac{d_i^{'+}}{\sum_{i=1}^{n} d_i^+}}{\left( \sum_{i=1}^{n} d_i^+ \right)^2},
\end{equation}

where $\sum_{i=1}^{n} d_i^+ = \sum_{i=1}^{n} h'(o_i - z) + h'(-o_i - z)$, $\sum_{i=1}^{n} d_i^- = \sum_{i=1}^{n} h'(o_i - z) - h'(-o_i - z)$, $\sum_{i=1}^{n} d_i^{'+} = \sum_{i=1}^{n} h''(o_i - z) + h''(-o_i - z)$, and  $\sum_{i=1}^{n} d_i^{'-} = \sum_{i=1}^{n} h''(o_i - z) - h''(-o_i - z)$.

Now we prove $\nabla^2 \phi(O)$ is PSD by showing 
$\boldsymbol{x} \nabla^2 \phi(O) \boldsymbol{x}^T$
is always greater than or equal to zero for an arbitrary vector $\boldsymbol{x}$.
\begin{align}
\boldsymbol{x} \nabla^2 \phi(O) \boldsymbol{x}^T & = \sum_{i=1}^{n} \sum_{j=1}^{n} \frac{\partial^2 z}{\partial o_i \partial o_j} x_i x_j = \sum_{i=1}^{n} x_i \sum_{j=1}^{n} \frac{\partial^2 z}{\partial o_i \partial o_j} x_j \\ \nonumber
& = \frac{1}{\left( \sum_{i=1}^{n} d_i^+ \right)^2} \left( \left(  \frac{\sum_{i=1}^{n}d_i^{'+}}{\sum_{i=1}^{n} d_i^+} \right) \left( \sum_{i=1}^{n} x_i d_i^- \right)^2 + \left( \sum_{i=1}^{n} d_i^+ \right) \left( \sum_{i=1}^{n} x_i^2 d_i^{'+} \right)\right) \\ \nonumber
& - \frac{2}{\left( \sum_{i=1}^{n} d_i^+ \right)^2}  \left( \sum_{i=1}^{n} x_i d_i^{'-} \right) \left( \sum_{i=1}^{n} x_i d_i^- \right).
\end{align}

To prove the above expression is greater than or equal to zero, notice that the denominator and $\sum_{i=1}^{n} d_i^+$ are always greater than zero, so we just need to prove
\begin{equation}
    \left( \sum_{i=1}^{n} d_i^{'+}\right) \left( \sum_{i=1}^{n} x_i d_i^- \right)^2 -  2\left( \sum_{i=1}^{n} d_i^{'+}\right)\left( \sum_{i=1}^{n} x_i d_i^{'-} \right) \left( \sum_{i=1}^{n} x_i d_i^- \right) + \left( \sum_{i=1}^{n} d_i^+ \right)^2 \left( \sum_{i=1}^{n} x_i^2 d_i^{'+} \right)  \ge 0.
\end{equation}
We first complete the square using the first two terms, so the left hand side is equal to
\begin{align}
\left( \sum_{i=1}^{n} d_i^{'+} \right) \left( \left( \sum_{i=1}^{n} x_i d_i^- \right) - \frac{\left( \sum_{i=1}^{n} x_i d_i^{'-} \right) \left( \sum_{i=1}^{n} d_i^+ \right)}{\sum_{i=1}^{n} d_i^{('+')}} \right)^2 
+ \left( \sum_{i=1}^{n} d_i^+ \right)^2 \left( \sum_{i=1}^{n} x_i^2 d_i^{'+} \right) \\\nonumber
- \left( \sum_{i=1}^{n} d_i^{'+} \right) \left( \frac{\left( \sum_{i=1}^{n} x_i d_i^{'-} \right) \left( \sum_{i=1}^{n} d_i^+ \right)}{\sum_{i=1}^{n} d_i^{'+}} \right)^2.
\end{align}
To prove the last two terms are greater than or equal to zero,  they are first simplified to 
\begin{equation} \label{eqn:convexity variance like term}
\frac{\sum_{i=1}^{n} x_i^2 d_i^{'+}}{\sum_{i=1}^{n} d_i^{'+}} - \left(\frac{ \sum_{i=1}^{n} x_i d_i^{'-} }{ \sum_{i=1}^{n} d_i^{'+}} \right)^2 \geq 0.
\end{equation}
We can treat $\frac{d_i^{'+}}{\sum_{i=1}^{n} d_i^{'+}}$ as the probability for $x_i$. Then from $\text{Var}(|X|) = \mathcal{E}[|X|^2] - \mathcal{E}[|X|]^2 \ge 0$, Cauchy-Schwarz and triangular inequality we get 
\begin{equation}
    \frac{\sum_{i=1}^{n} x_i^2 d_i^{'+}}{\sum_{i=1}^{n} d_i^{'+}} \ge \left(\frac{\sum_{i=1}^{n} |x_i| d_i^{'+}}{\sum_{i=1}^{n} d_i^{'+}}\right)^2 \ge \left(\frac{\sum_{i=1}^{n} |x_i| |d_i^{'-}|}{\sum_{i=1}^{n} d_i^{'+}}\right)^2 \ge \left(\frac{\sum_{i=1}^{n} x_i d_i^{'-}}{\sum_{i=1}^{n} d_i^{'+}}\right)^2,
\end{equation}
which indicates that~\ref{eqn:convexity variance like term} is always true. Therefore it is proved that the Hessian of $\phi(O)$ is positive semi-definite, and $\phi(O)$ is thus convex.
\end{proof}

\section{Estimating correlations using the MP-function}
\label{subsec:MP correlator}
In this subsection, we first formally define the MP correlator function and then prove that its expected output function is a monotonically increasing function of the input correlation for jointly Gaussian distributions utilizing Lemma~\ref{lem3} and the Price theorem.

\begin{Def}\label{Definition of MP correlator function}
For vectors $\boldsymbol{x} = \left[x_i\right], \boldsymbol{y} = \left[y_i\right], i = 1 ...N$, where $x_i$ and $y_i$ are independent and identically distributed (i.i.d.) samples from jointly Gaussian distributed random variables $X$ and $Y$ with following probability density function
\begin{equation}
    	p(X=x,Y=y; R) = \frac{1}{2\pi\sqrt{1-R^2}} \exp \left[-\frac{x^2 + y^2 - 2R x y}{2\left(1- R^2\right)} \right],
\end{equation}
where $R \in (-1,1)$ is the Pearson's correlation between $X$ and $Y$, we define the MP correlator function for a given constraint $\gamma$ as 
\begin{equation}\label{eqn:MP correlator function}
   f(\boldsymbol{x},\boldsymbol{y}) 
   = \phi(O^+; \gamma^+) - \phi(O^-; \gamma^-),
\end{equation}
where $O^+ = [x_i+y_i,-x_i-y_i]$, $O^- = [x_i-y_i,y_i-x_i]$ are the positive and negative operands,
$z = \phi(O; \gamma)$ is given by the MP function in~(\ref{eqn:MPbasic}). Note that $\gamma^+$ is not necessarily equal to $\gamma^-$.
\end{Def}

In the next theorem, we state that the expected output of the MP correlator function defined above increases monotonically with respect to the correlation $R$ between inputs.
\begin{Thm}\label{Thm1} The expected output $\mathcal{E}[f(\boldsymbol{x},\boldsymbol{y})]$ of MP correlator function defined in~\ref{eqn:MP correlator function} is a monotonically increasing function $G_{MP}(R)$ of the correlation $R$ between jointly normal distributed inputs $X$ and $Y$.
\end{Thm}
\begin{proof}
To show this, we need to prove that $\frac{d G_{MP}(R)}{d R} = \frac{\partial E\left[f\right]}{\partial R}$ is always greater than 0. Because $x_i, y_i$ are i.i.d. samples from unit variance jointly normal variates, the covariance matrix $[\rho_{ij}]$ representing the correlations between $x_i$ and $y_j$ will be a diagonal matrix where 
\begin{equation}
\rho_{ii} = R \text{ and } \rho_{ij} = 0, \text{ for } i\neq j. 
\end{equation}
On the other hand, the correlations between $x_i$ and $x_j$ and the correlation between $x_i$ and $y_j$ for $i \neq j$ are always zero, and the correlations between $x_i$ and itself and the correlations between $y_i$ and $y_j$ are always one. As a result, the expected output function should not be a function of these correlations.
According to the Price theorem, we have 
\begin{equation}
     \frac{\partial E\left[f\right]}{\partial R} = \sum_i \frac{\partial E\left[f\right]}{\partial \rho_{ii}} = \sum_i E\left[\frac{\partial^2 f}{\partial x_i\partial y_i}\right],
\end{equation}
where
\begin{equation} 
\label{eqn:total derivative}
    \frac{\partial^2 f}{\partial x_i\partial y_i} = \frac{\partial^2 \phi(O^+)}{\partial x_i\partial y_i} - \frac{\partial^2 \phi(O^-)}{\partial x_i\partial y_i}.
\end{equation}
Using the chain rule, we have 
 \begin{equation}
     \frac{\partial}{\partial x_i} \phi(O^+) = \sum_{k=1}^{n} \frac{\partial \phi}{\partial o^+_k} \cdot \frac{\partial o^+_k}{\partial x_i} = \frac{\partial \phi}{\partial o^+_i},
 \end{equation}
where $o^+_k = x_k + y_k$. This leads to 
 \begin{equation}\label{eqn:pos double deriv}
\frac{\partial^2 \phi(O^+)}{\partial x_i \partial y_i}
= \frac{\partial}{\partial y_i} \left( \frac{\partial \phi}{\partial o^+_i} \right)
= \sum_{k=1}^{n} \frac{\partial^2 \phi}{\partial o^+_i \partial o^+_k} \cdot \frac{\partial o^+_k}{\partial y^+_i}
= \frac{\partial^2 \phi}{\partial (o^+_i)^2}.
 \end{equation}
Similarly, it can be shown that 
\begin{equation}\label{eqn:neg double deriv}
\frac{\partial^2 \phi(O^-)}{\partial x_i \partial y_i}
= -\frac{\partial^2 \phi}{\partial (o^-_i)^2}.
\end{equation}
From the convexity of $\phi$, we have ~\ref{eqn:pos double deriv} $>0$ and~\ref{eqn:neg double deriv} $<0$, substituting them into~\ref{eqn:total derivative}, we have 
 \begin{equation} \label{eqn:total nonnegative derivative}
    \frac{\partial E\left[f\right]}{\partial R} \ge 0.
\end{equation}
\end{proof}
Due to Theorem~\ref{Thm1}, the calibration function $G_{MP}^{-1}$ can be well defined, allowing the construction of the energy-efficient analog correlator using the MP function.

\begin{figure}[htbp]
  \centering
\includegraphics[width=1.0\textwidth]{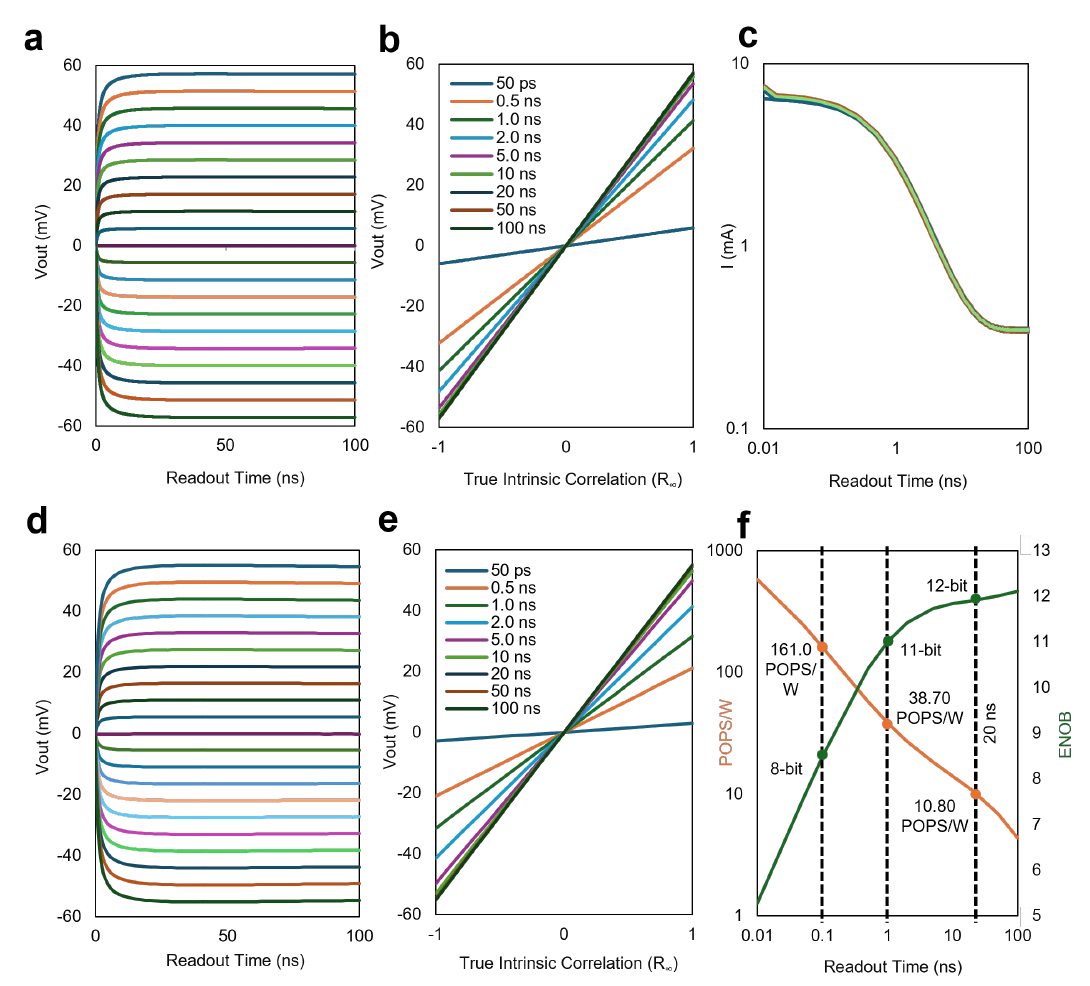}
  \caption{Simulation results for an MP correlator with length N=1024, $R = \SI{3000}{\ohm}$ and $C = \SI{50}{\pico\farad}$. (a) Time evolution of output voltages under varying correlation levels in numerical simulations. (b) Output voltages at selected time points for different correlation values in numerical simulations. (c) Current dynamics in the positive and negative branches across different correlation levels in circuit simulations. (d) Time evolution of output voltages in circuit simulations for varying correlations. (e) Output voltages at selected time points for different correlation values in circuit simulations. (f) Performance metrics—peta operations per second per watt (POPS/W) and effective number of bits (ENOB)—versus readout time. }
  \label{fig:behavioral simulations}
\end{figure}

\section{Simulations of MP-correlator transient dynamics }
\label{sec:behavioral simulations with cap}
This section explores the dynamics of MP correlators when the correlation is measured while the MP dynamics has not reached its steady-state. In practical hardware implementation, the correlator cannot reach the steady-state solution instantaneously, due to parasitic circuit elements and interconnect capacitance. Usually, the readout is performed after the output converges, that is, waiting about $5RC$ constant time. However, this section demonstrates the ability of MP correlators to make readouts before steady states, through both numerical and circuit simulations. We show that with a shorter readout time, the MP correlator can reach higher throughput and energy efficiency, achieving peta operations per second (POPS) and POPS per watt (POPS/W).

To study how parasitics in hardware could affect the time domain evolution of the MP correlator output, we model the dynamics as follows
\begin{subequations}
\begin{align}
    \sum_{n=1}^{N}I_{DS}(V_o + x_i,V_o + y_i, V_s^+(t))+I_{DS}(V_o - x_i,V_o - y_i,V_s^+(t))=\frac{V_s^+(t)}{R}+C \frac{dV_s^+(t)}{dt},
\\
        \sum_{n=1}^{N}I_{DS}(V_o + x_i,V_o - y_i, V_s^-(t))+I_{DS}(V_o - x_i,V_o + y_i,V_s^-(t))=\frac{V_s^-(t)}{R}+C \frac{dV_s^-(t)}{dt},
\end{align}
\end{subequations}
where $I_{DS} (V_G, V_B , V_s)$ is the drain-to-source current of a transistor implementing, C is the parasitic capacitance, and R is the tunable resistance. In this case, the constraint of the MP function $\gamma(t) = \frac{V_s(t)}{R}+C \frac{dV_s(t)}{dt}$ is a function of time as well.
By linearization, we can show that 
\begin{equation}
        \frac{dV_s(t)}{dt} = \frac{1}{C}\left(\sum_{n=1}^{N}I_{DS}' V_s(t) - \frac{1}{R}\right),
\end{equation}
where $I_{DS}' < 0$ is the derivative of transistor currents with respect to the source voltage. It can be seen that the system is stable and $V_s$ is guaranteed to converge.

This dynamical system is numerically simulated and confirmed by circuit simulations for sinusoidal inputs.  The simulation results for a MP correlator of length $N = 1024$, with $R = \SI{3000}{\ohm}$ and $C = \SI{50}{\pico\farad}$ are shown in Fig.~\ref{fig:behavioral simulations}. Fig.~\ref{fig:behavioral simulations}a and d depict the output voltage dynamics over time for different input correlation levels, obtained from numerical and circuit simulations, respectively. Fig.~\ref{fig:behavioral simulations}b and e provide a zoomed in comparison of the output voltages at specific time instants across various correlations between the two simulations. The results demonstrate that the model qualitatively captures the key transient behaviors during the system's evolution.

Fig.~\ref{fig:behavioral simulations}c shows the temporal evolution of currents in the positive and negative branches across various correlation levels, reflecting the behavior of the constraint $\gamma(t)$. Although some initial fluctuations are present, the variations are minimal, and 
$\gamma(t)$ rapidly stabilizes. This convergence allows the output voltage to preserve a monotonic relationship with the correlation during the transient phase, consistent with the theoretical results in Section~\ref{subsec:MP correlator}.

Fig.~\ref{fig:behavioral simulations}f illustrates the variation of performance metrics with respect to readout time. While the output voltage and ENOB stabilize after \SI{20}{\nano\second}, achieving 12-bit precision and 10.8 POPS/W, the readout time can be reduced to as low as \SI{1}{\nano\second} or even \SI{0.1}{\nano\second} to enhance energy efficiency—yielding 11-bit and 8-bit precision with corresponding efficiencies of 38.7 POPS/W and 161.0 POPS/W, respectively. This improvement, however, comes at the cost of a reduced voltage swing. Similar observations are made in hardware measurements as shown in Fig.~E3g in the main text, where the measured ENOB is limited to 8.5 bits due to the resolution of the oscilloscope used during the hardware measurements.

\section{Estimating transient Signal Processing Gain }
Maintaining the 3dB signal processing gain for doubling the ensemble size is crucial for the scalability of the MP approach. In the main text, numerical simulations demonstrate that the MP collective computing approach can achieve the same 3dB processing gain at different time steps in its dynamics as the MAC approach for ensemble sizes $N$ from $2^6$ to $2^{18}$. In this section, we describe the process by which the simulation is conducted. 
The dynamical system is modeled as follow
\begin{subequations}\label{eqn:signal_processing_gain_dynamics}
\begin{align}
    \sum_{n=1}^{N}\left[x_i + y_i - z^+(t)\right]_+ +\left[-(x_i + y_i) - z^+(t)\right]_+ = \frac{z^+(t)}{R}+C \frac{dz^+(t)}{dt},
\\
    \sum_{n=1}^{N}\left[x_i - y_i - z^-(t)\right]_+ +\left[-(x_i - y_i) - z^-(t)\right]_+ = \frac{z^-(t)}{R}+C \frac{dz^-(t)}{dt},
\end{align}
\end{subequations}
which captures the same features in the hardware model but uses the ReLU function for nonlinearity $h$. Here, inputs $x_i$ and $y_i$ are correlated Gaussian pseudo-random variables, $R$ and $C$ are hyperparameters, where $R$ determines the static $\gamma$ in the steady state, $C$ controls the update rate, and the converging time constant $RC$. 
At steady state, this is equivalent to the steady state solution of the following maximum entropy function, 
\begin{equation}
    S_{2}^{\pm} + \alpha^{\pm}\sum{(p_i^{\pm} + q_i^{\pm})}+\beta^{\pm}\left(\sum{(x_i\pm y_i)(p_i^{\pm} - q_i^{\pm}) - U^{\pm}}\right), \\
\end{equation}
which gives us $z := \alpha/\beta$ and $\gamma := 2/|\beta| = z/R$. To ensure that the system has the same dynamic range for different ensemble sizes, we need to scale $\gamma$ linearly with $N$, that is, $R$ needs to go down. At the same time, $C$ can be linearly scaled up with $N$ to maintain the same convergence speed. In simulation, we choose $R = 25/N$ and $C$ = $10^{-5}N$ with a time step of $dt = 10^{-5}$. The above observation has the same implications for the design of a hardware system as well.

Note that the output trajectory and output patterns at different times displayed in~Fig.~\ref{fig:entropy_fig1}g and h are mainly dictated by the way $\beta$ or, equivalently, $\gamma$, changes. The update equation in~\ref{eqn:signal_processing_gain_dynamics} results in an output trajectory similar to that of hardware measurements, because both are moving in the direction of current or charge quantities in each device. Other Lagrangian schedules are possible to mimic the parasitic inductive effects, which result in a different trajectory. This
is discussed in more detail in the next section. It was not observed in the hardware measurements or circuit simulations, as the parasitic inductance is negligible.

\section{Simulating effect of inductance on MP dynamics}
In Section~4.1, the simulation method is discussed for one possible output trajectory of the maximum entropy optimization, which resembles actual hardware measurements and behavioral models as discussed in Section~\ref{sec:behavioral simulations with cap}. It is suggested that a more complex trajectory is possible if the parasitic inductance is not negligible. We model this by $L_{par}$ in between with the source of NMOS and the $R_{sink}$ and $C_{par}$, which gives a dynamical system below
\begin{subequations} \label{eqn:second_order_hardware_dynamics}
\begin{align}
  \sum_{i=1}^{N} I_{DS} (V_o+x_i,V_o\pm y_i,V_s^{\pm}(t) ) +I_{DS} (V_o\mp x_i,V_o-y_i,V_s^{\pm}(t))= \frac{V_{out}^{\pm}(t)}{R_{sink}}+C_{par}\frac{dV_{out}^{\pm}(t)}{dt},  \\
   L_{par}\sum_{i=1}^{N} \frac{dI_{DS} (V_o+x_i,V_o\pm y_i,V_s^{\pm}(t) )}{dt}+\frac{dI_{DS} (V_o\mp x_i,V_o-y_i,V_s^{\pm}(t))}{dt} = V_s^{\pm}(t) - V_{out}^{\pm}(t).
\end{align}
\end{subequations}
where $V_s^{\pm} - V_{out}^{\pm}$ denotes the voltage across the inductor.

Fig.~\ref{fig:RLC behavioral simulations} displays the output dynamics for sinusoidal inputs with different correlation levels in numerical simulations for the above second-order MP correlator model with 
 $N = 1024$, $R = \SI{1000}{\ohm}$, $C = \SI{100}{\pico\farad}$, and $L = \SI{2}{\nano\henry}$ in Fig.~\ref{fig:RLC behavioral simulations}a and b.  The maximum entropy system simulation for correlated Gaussian inputs with a similar dynamics using $h(.)=\text{max}(0,.)$ is shown in  Fig.~\ref{fig:RLC behavioral simulations}c and d.

The similarities of the output trajectory in Fig.~\ref{fig:RLC behavioral simulations}a and c demonstrate the flexibility of the maximum entropy framework as a model of the MP correlator by changing its update schedule of the Lagrangian multiplier. The differences in output functions in Fig.~\ref{fig:RLC behavioral simulations}b and d arise from the difference in the nonlinear functions adopted in the simulations (ReLU for maximum entropy optimization and the saturation current function of MOSFETs for circuit simulations). 
\begin{figure}[htbp]
  \centering
\includegraphics[width=1.0\textwidth]{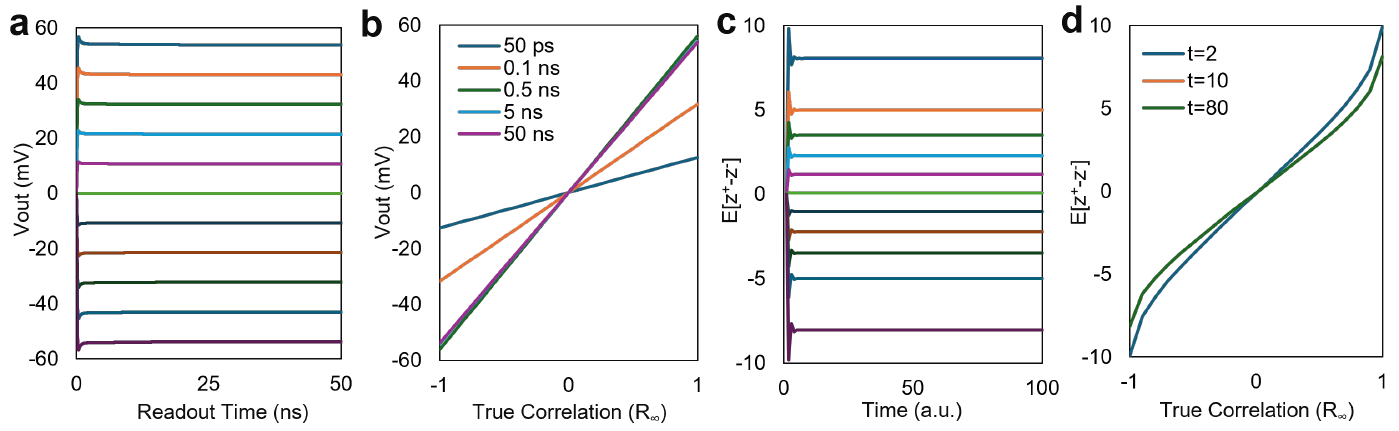}
  \caption{Another possible output dynamics for the maximum entropy simulation and MP correlators with parasitic inductance. (a) Output dynamics across different correlation levels of MP correlators with parasitic inductance. (b) Output voltages of MP correlator simulations at selected time points for different correlation values in circuit simulations. 
  (c) Dynamics of the expected differential output of the maximum entropy simulations with $h(.) = \text{max}(0, .)$.  (d) The expected differential outputs of the maximum entropy simulation at different time-steps, as the correlation is varied. } 
  \label{fig:RLC behavioral simulations}
\end{figure}

\section{Performance comparison with other hardware platforms}
Different hardware platforms are optimized for specific computational workloads, and their I/O (input/output) subsystems are designed accordingly. This means I/O characteristics must be taken into account when comparing the performance of such systems. For example, GPUs typically use memory types like LPDDR~\cite{micron_lpddr5x_2023} or HBM~\cite{micron_hbm3e_2023}, which are designed for digital data and offer high aggregate bandwidth. However, the bandwidth per I/O pin is relatively low—around 1 GB/s as shown in Fig.~\ref{fig:IO compare}—due to the digital nature of the data. A similar trend is observed in in-memory computing (IMC) platforms, where the required bandwidth may still be high, but individual I/Os are slower, since weights are embedded locally within the memory-compute units.

In contrast, RF correlators, including this work, require significantly higher I/O speed due to the analog nature of RF signals. Additionally, achieving high resolution in RF processing demands high sample precision. Consequently, RF correlators impose stricter requirements on I/O speed than conventional digital computing systems. For instance, MP systems can support I/O rates as high as 10 GS/s with 8-bit resolution per sample.
\begin{figure}[H]
  \centering
\includegraphics[width=0.5\textwidth]{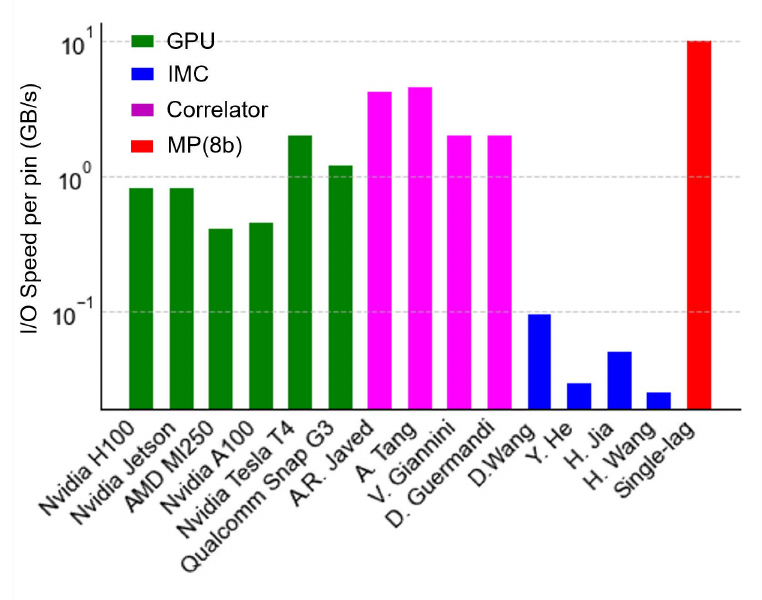}
  \caption{Comparison of I/O speeds (normalized to per pin) between different hardware computing platforms.}
  \label{fig:IO compare}
\end{figure}

\section{Correlator characterization}

The RF correlator ICs were evaluated across a range of experimental setups to assess their performance under both controlled laboratory conditions and representative real-world use cases. A central figure of merit in this evaluation is the hardware dynamic range (HDR), which reflects the effective computational precision of the IC. In correlation-based systems, the total error originates from two main sources: (1) the finite length of the input data window, and (2) inherent hardware computation inaccuracies. When pure sinusoidal waveforms are used as inputs, the error due to finite-length sampling becomes negligible, allowing isolation and direct assessment of hardware-induced errors.

To carry out this characterization, two coherent sinusoidal signals were generated using a Tabor Electronics Proteus Arbitrary Waveform Generator (P9484D), while a Keysight analog signal generator (E8257) supplied the local oscillator (LO) signal used for generating the sampling clock. The correlator’s output was digitized using the high-speed acquisition capability of the Arb. By systematically sweeping the relative phase between the two input signals and recording the corresponding correlation outputs, the deviation from the ideal correlation was computed as shown in Fig.3b of main text. The root-mean-square (RMS) value of this deviation yields the HDR, which is then used to estimate the effective number of bits (ENOB) using the relation: $ENOB = (\mathrm{HDR} - 1.76)/6.02$.

The RF correlator ICs were evaluated across a range of input sinusoidal frequencies and sampling rates to assess the stability and robustness of performance metrics. Each experiment was repeated multiple times, with the first 50 runs used to train the inverse mapping function $G^{-1}$, and the subsequent 100–200 runs reserved for performance testing. The reported hardware dynamic range (HDR) and effective number of bits (ENOB) represent the mean values computed across the test runs. The correlator consistently demonstrated stable HDR performance across varying input frequencies and sampling rates. Based on the ENOB and sampling and compute power consumption, the TOPS/W is computed using eqn. \ref{eqn:topsw_sys} or \ref{eqn:topsw_compute}.

\begin{equation}
    \text{TOPS/W} (\text{System}) = \frac{N \times (\text{ENOB}^2 + \text{ENOB})}{10^{12} \times (E_{\text{Sampling}} + E_{\text{Compute}})} \label{eqn:topsw_sys}
\end{equation}

\begin{equation}
    \text{TOPS/W} (\text{Compute}) = \frac{N \times (\text{ENOB}^2 + \text{ENOB})}{10^{12} \times  E_{\text{Compute}}} \label{eqn:topsw_compute} 
\end{equation}
where,
\begin{align*}
    E_{\text{Sampling}} &= P_{\text{DC, Sampling}} \times \text{Sampling Time}, 
\end{align*}
and the compute energy $E_{\text{Compute}}$ is estimated as detailed in the methods section of the main text. As described in the main text, the IC was further characterized by varying the readout time to explore trade-offs between precision and throughput. The optimal readout time for duty-cycled operation was selected to ensure a minimum ENOB of 8 bits, thereby maintaining high computational fidelity while maximizing data throughput. Based on these measurements, a 20 ns readout time was identified as the optimal point, enabling at least 8-bit ENOB performance while achieving compute efficiencies in the range of several Peta Operations per Watt.

\begin{figure}[htbp]
  \centering
\includegraphics[width=1\textwidth]{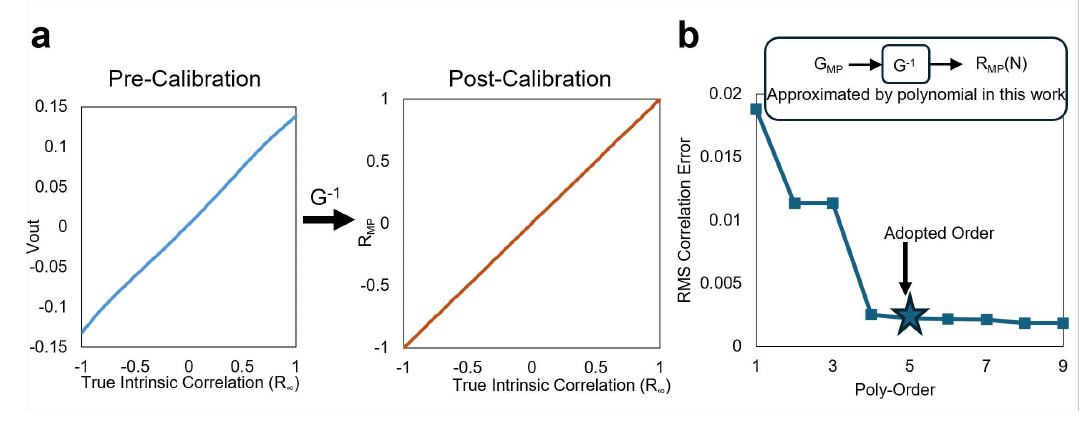}
  \caption{(a) Calibration process for Periodic Inputs. (b) Training correlation error across polynomial order.}
  \label{fig:cal}
\end{figure}
\begin{figure}[H]
  \centering
\includegraphics[width=1\textwidth]{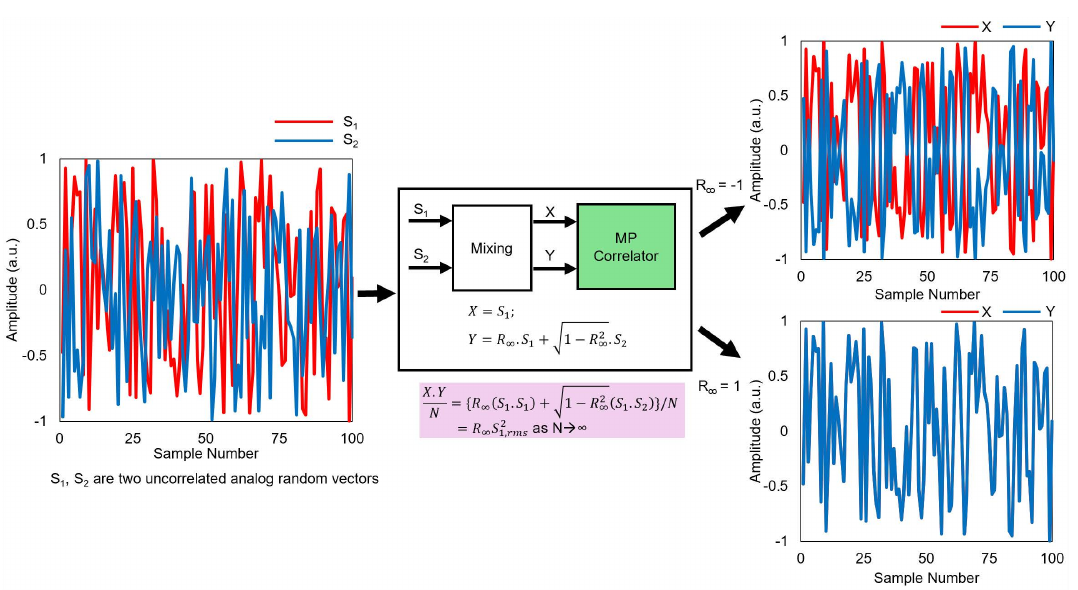}
  \caption{Correlated random inputs generation process.}
  \label{fig:mix}
\end{figure}

\subsection{Calibration procedure and estimating correlations}
As discussed earlier, when applications require accurate estimation of the true correlation or output equivalent to a series of MAC (multiply-and-accumulate) operations, a post-calibration step is necessary. This calibration involves learning a nonlinear mapping function, denoted as $G^{-1}$, that maps the measured correlator output to the true correlation value. The mapping function $G^{-1}$ is application-specific and must be performed once per use-case, unless the operating conditions or signal characteristics change significantly.

During the calibration process, the IC is trained using known ground-truth inputs. One example of such calibration is relative phase estimation between periodic sinusoidal inputs. During training, input signals with phase offsets ranging from $0^\circ$ to $180^\circ$ in $1^\circ$ steps are applied, capturing the full range of correlation values. The corresponding output of the correlator forms a monotonic response spanning approximately \SI{-50}{\milli\volt} to \SI{+50}{\milli\volt}. To map this response to the true correlation range of $[-1, 1]$, a polynomial regression model ($a_0 + a_1 x + a_2 x^2 + .. + a_nx^n$) is trained as shown in Fig.\ref{fig:cal}a. A higher-order polynomial (order $>1$) is required due to the inherent compressive nature of the MP correlator, which causes nonlinear behavior near extreme correlation values (i.e., near $-1$ and $+1$). In Fig.\ref{fig:cal}b, it can be observed that the polynomial training error doesn't improve much after the $5^{th}$ order polynomial.

Similarly, for the esitmating the signal processing gain (SPG), uniformaly distributed random vectors are used. The correlator is trained by injecting two analog vectors $X=S_1$ and $Y=R_{\infty} S_1+\sqrt{1-R_{\infty}^2}S_2$ with varying correlations $R_{\infty}$ and seed, where $S_1$ and $S_2$ are uncorrelated uniform distributed analog random vectors as shown in Fig.\ref{fig:mix}. The output of the correlator is then mapped one-to-one by training polynomial coefficients to the true input correlation. Here, the training error is limited by a finite correlation length. 

This calibration strategy ensures accurate recovery of correlation values and extends the applicability of the MP correlator to a wide variety of real-world signal processing tasks.

\subsection{Signal detection without calibration}
Many RF applications does not require an extensive calibration, and a simple thresholding-based approach suffices. Applications such as spectrum sensing Fig.\ref{fig:thres}a, wireless communication Fig.\ref{fig:thres}b, and code-domain radar Fig.\ref{fig:thres}c inherently produce sparse or "thumbtack"-like correlation responses. In these cases, the correlator output contains sharp peaks corresponding to spectral content (in spectrum sensing), code matches (in wireless communication), or object reflections (in radar). The sparsity of the output allows for direct detection using thresholding techniques without requiring precise mapping to ground-truth values. Therefore, in such scenarios, the MP correlator IC can be directly deployed without the need for any post-calibration. This greatly simplifies system integration and enables real-time, low-power signal detection in dynamic RF environments.

\begin{figure}[htbp]
  \centering
\includegraphics[width=1\textwidth]{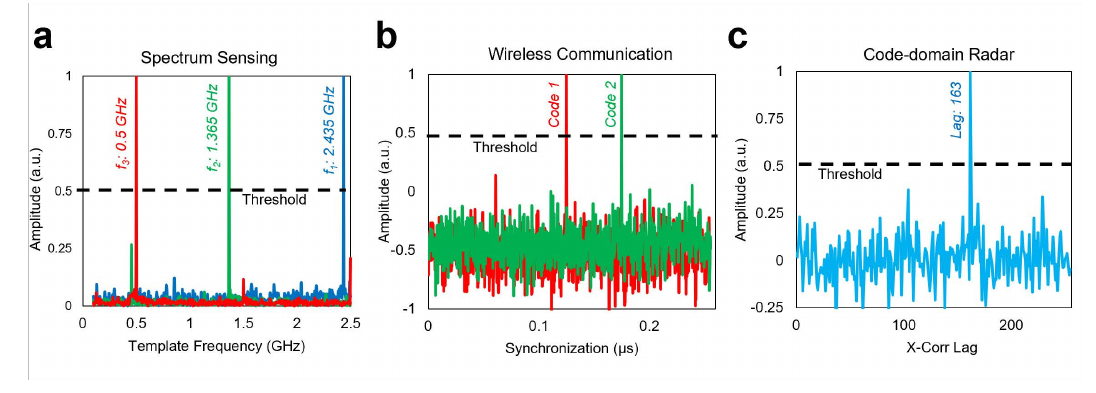}
  \caption{Thresholding-Based Approach. (a) Spectrum Sensing. (b) Wireless Communication (c) Code-domain Radar.}
  \label{fig:thres}
\end{figure}

\section{Process-Voltage-Temperature (PVT) results}
As described in the main text, the MP compute unit cell employs dihedral symmetry and an adaptive source-follower biasing scheme that enhances the robustness of correlation computation against process, voltage, and temperature (PVT) variations. Figure~\ref{fig:pvt}a illustrates the minimal chip-to-chip variation, demonstrating strong immunity to process-induced mismatches. In Figure~\ref{fig:pvt}b, the HDR remains largely unaffected by fluctuations in supply voltage, highlighting the circuit’s voltage robustness. In Fig.~\ref{fig:pvt}c, the measured signal processing gain (here, measurement results are presented from a similar prior IC~\cite{Kareem_JSSC2024})also exhibits negligible variation across a wide temperature range, indicating thermal stability.

\begin{figure}[H]
  \centering
\includegraphics[width=1\textwidth]{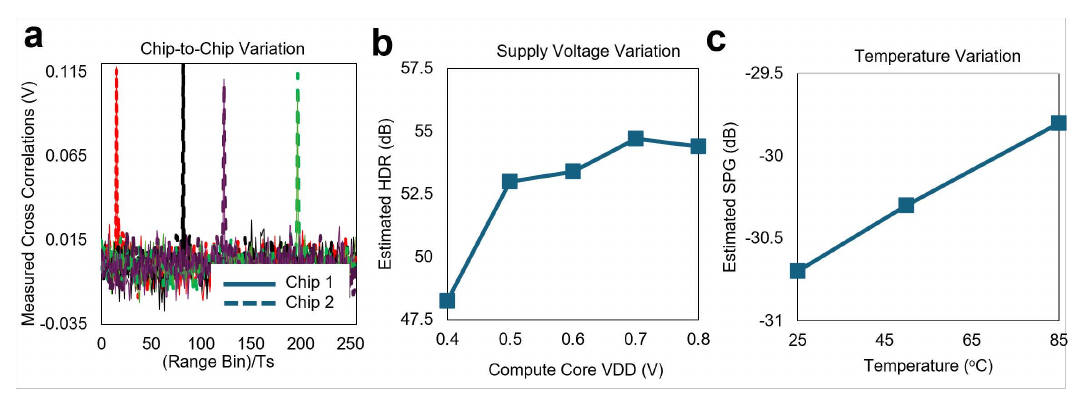}
  \caption{PVT Robustness. (a) Chip-to-chip variation. (b) Supply voltage variation. (c) Temperature variation.}
  \label{fig:pvt}
\end{figure}

\section{Detailed experimental setup}
To characterize and evaluate the performance of the RF correlator ICs, a dedicated test setup was developed as shown in Fig.\ref{fig:setup}. The IC requires a local oscillator (LO) signal to drive its sampling clock, which was supplied by a Keysight PSG analog signal generator. A custom-designed DC biasing board incorporating low-dropout regulators (LDOs) was used to generate the required bias voltages, with power delivered via a Keysight programmable DC power supply. Functional input signals and emulated RF environments were generated using a Tabor Proteus arbitrary waveform generator (Arb). Sub-20 microcontroller was used to program the on-chip SPI.

The output from the IC was amplified using an on-chip differential-to-single-ended buffer and captured using a Rigol oscilloscope. Input and template signals—tailored to the specific application scenario—were generated using the Arb. Due to the 8-bit effective number of bits (ENOB) of the readout oscilloscope, most of the characterization results presented are limited to this resolution ceiling.

For IC characterization, rail-to-rail sinusoidal signals ranging from 0 to 0.8V were applied at the input, with frequency and phase systematically varied to evaluate the correlator's response. In the spectrum sensing experiments, noisy input signals containing single or multiple tones were used, while the template signal was a frequency-swept sinusoid spanning from DC to the Nyquist frequency. These template signals incorporated both in-phase (I) and quadrature (Q) components, enabling the extraction of the input signal’s relative phase through correlation. This capability is particularly relevant for applications such as angle-of-arrival (AoA) estimation.

In the compressive sensing experiment, the input signal was constructed as a sparse signal in the frequency domain. The template signals were generated by selecting K ($\ll$N) random rows of the Discrete Fourier Transform (DFT) matrix and modulating them with pseudo-random sequences. The spectrum of the input signal was reconstructed using the CoSaMP (Compressive Sampling Matching Pursuit) algorithm, which operates under the assumption of sparsity in the input signal.

In the code-domain communication experiments, the input signal consisted of multiple pseudo-random codes modulated onto a 2GHz carrier in a noisy environment with a signal-to-noise ratio (SNR) of 0dB. The correlator simultaneously performed analog downconversion and code-domain demodulation to recover the transmitted codes. The successful recovery of the desired codes in the presence of strong interference highlights the robust selectivity and interference resilience of the proposed system architecture.

In the code-modulated 64 APSK system, a pseudo-random code was mixed with periodic sine and cosine carrier signals with varying amplitudes and phases to generate appropriate constellation points. Then, a 0dB white Gaussian noise and blocker codes are added to the original signal. This noisy signal is then presented as an input to the correlator and I and Q phase periodic signals are used as templates to downconvert and demodulate to extract the codes.

\begin{figure}[H]
  \centering
\includegraphics[width=1\textwidth]{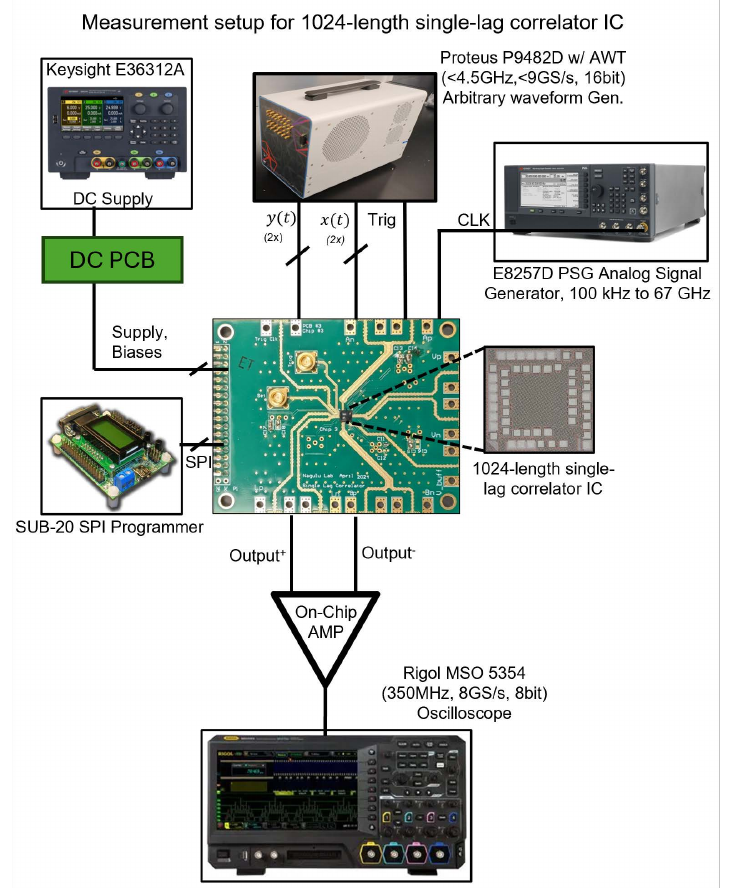}
  \caption{Measurement setup for 1024-length single-lag correlator IC.}
  \label{fig:setup}
\end{figure}






\end{document}